\documentclass{article}

\usepackage[preprint]{neurips_2026}
\usepackage{fancyhdr}
\fancypagestyle{publicationtitle}{\fancyfoot{}}
\AtBeginDocument{\pagestyle{fancy}\setlength{\headwidth}{\textwidth}}
\renewcommand{\acksection}{\section*{Acknowledgments}}

\usepackage[utf8]{inputenc} 
\usepackage[T1]{fontenc}    
\usepackage{hyperref}       
\usepackage{url}            
\usepackage{booktabs}       
\usepackage{float}          
\usepackage{amsfonts}       
\usepackage{amssymb}        
\usepackage{nicefrac}       
\usepackage{microtype}      
\usepackage{xcolor}         
\usepackage{algorithm}      
\usepackage{algorithmic}    
\usepackage{amsmath}        
\usepackage{amsthm}         
\usepackage{graphicx}       
\usepackage{subcaption}     
\definecolor{citeblue}{RGB}{45, 105, 160}
\setcitestyle{numbers,square,comma}
\hypersetup{
  colorlinks=true,
  citecolor=citeblue,
  linkcolor=citeblue,
  urlcolor=citeblue
}
\usepackage{makecell}
\newtheorem{proposition}{Proposition}
\newtheorem{lemma}{Lemma}

\title{\textit{GibbsLens}: Training-Free Camera Control via Block-Conditional Gibbs Refinement}

\author{%
  Yuzhu Wang\thanks{Equal contribution.}\textsuperscript{\hspace{0.6em}1,2} \quad
  Xi Ye\footnotemark[1]\textsuperscript{\hspace{0.6em}1} \quad
  Duo Su\textsuperscript{1} \quad
  Yangyang Xu\textsuperscript{1} \quad
  Jun Zhu\textsuperscript{1}\thanks{Corresponding author.} \\
  \textsuperscript{1}Department of Computer Science and Technology, Tsinghua University \\
  \textsuperscript{2}South China University of Technology \\
}

\begin{document}

\maketitle
\thispagestyle{publicationtitle}

\begin{abstract}
Training-free camera control for pretrained flow-matching video generators is
a partial-observation inverse problem: a depth-warped guidance video supplies
noisy evidence on a subset of latent sites, which the sampler must reconcile
with the pretrained prior. Existing methods struggle to balance the trade-off between trajectory
adherence and visual quality and the heuristic guidance-strength
tuning lacks robustness. We propose \textbf{\textit{GibbsLens}}, which resolves
this dilemma through a structural change to the sampler: each outer
hard-replacement guidance step is augmented with an inner-loop
\emph{block-conditional pseudo-Gibbs refinement} on the unobserved complement
at the same noise level, with provable convergence to the partial-observation
conditional data law. To accelerate convergence on high-dimensional video latents, we exploit
their conditional locality, partitioning the unobserved complement into 3D
patches, each tracked by a custom mixing indicator that adaptively freezes
converged patches. On RealEstate10K and DAVIS,
\textbf{\textit{GibbsLens}} attains the best FVD against all seven training-free and
training-based competitors, outperforming every training-free baseline on
every reported metric. Project page: \url{https://xiye20.github.io/projects/GibbsLens/}.
\end{abstract}

\section{Introduction}
\label{sec:intro}


Reliable camera control is essential in filmmaking, simulation, virtual reality, and embodied AI, where the desired camera trajectory is given at inference time.
The conventional approach involves retraining a video generation model to follow trajectory instructions, which is costly~\cite{bai2025recammaster}.
Training-free camera control sidesteps this cost but remains difficult even for strong modern video diffusion and flow-matching models. 
During sampling, camera motion is entangled with object motion, scene layout, and appearance synthesis, making it much harder to prescribe a target viewpoint trajectory than to specify text semantics alone.

Specifically, existing camera-control methods can be organized into two paradigms.
\emph{Training-based} methods encode camera trajectories and fine-tune the generator to learn the mapping between these control signals and the output video~\citep{he2024cameractrl,he2025cameractrl,bahmani2024vd3d,bahmani2025ac3d,bai2025recammaster}. 
While effective, it requires model-specific training pipelines and relies on paired trajectory–video datasets that are expensive to acquire.
\emph{Training-free} methods instead synthesize a warped guidance video from the desired trajectory geometry and inject it as inference-time evidence~\citep{zhang2025recapture,hou2024training,zhou2025latent}, sidestepping the training cost entirely.

Essentially, training-free conditional generation is an inverse-problem question. 
Two competing forces arise during the sampling process.
The \emph{posterior} (trajectory codes) pulls the sample toward the conditioning evidence, while the \emph{prior} (pretrained models) shapes it according to the learned distribution~\citep{lu2023contrastive,he2023manifold,chung2022diffusion}. 
Tilting toward the posterior tightens adherence to the evidence but exposes the sample to noise and artifacts in the observation; tilting toward the prior preserves visual quality but loosens adherence. 
Existing training-free samplers balance these forces through guidance strength (the inverse temperature in energy guidance, or the gradient step size in DPS-style samplers~\citep{chung2022diffusion,ye2024tfg}). 
Tuning this scalar is heuristic and lacks robustness across settings.

Training-free camera control is a representative instance of this dilemma. 
Hard latent replacement~\citep{lugmayr2022repaint,singer2025time,song2025worldforge,wang2026coarse} operates at maximum guidance strength, treating the warped video as noise-free evidence and overwriting the masked latent regions at every denoising step. 
This delivers tight trajectory control; however, the warp is inherently imperfect due to depth-estimation errors, projection artifacts, and downsampled visibility, so the strong observation introduces unstable hallucinated content and artifacts along mask boundaries. 
Soft posterior guidance~\citep{chung2022diffusion} and windowed guidance~\citep{kynkaanniemi2024applying} relax the guidance, recovering visual quality at the cost of trajectory adherence.

Rather than retuning the guidance strength, we propose \textit{GibbsLens} to resolve the dilemma through a structural change to the sampler.
Started with a hard masked latent replacement at the outer denoising step, which preserves tight trajectory control, we introduce a \emph{block-conditional pseudo-Gibbs refinement}: an inner loop that refines the prior on the unobserved region at the same noise level. 
The refinement extends Bengio's generalized denoising-auto-encoder Markov chain~\citep{bengio2013generalized} to a partial-observation conditional setting, and a sub-state ergodicity argument shows its iterates converge to the partial-observation conditional data law. 
To make this primitive practical at video scale, we leverage the property of latent flow-matching models that video latents exhibit short-range conditional dependence along all spatial-temporal axes. 
Building on this locality, we partition the unobserved complement into 3D patches. 
Each patch is then tracked by a mixing indicator that freezes it once its inner iterations have converged, accelerating convergence and improving overall sample quality.

Our contributions are three-fold:
\begin{enumerate}
    \item We construct a generalized denoising-auto-encoder chain on the unobserved complement, provably converging to the partial-observation conditional data law, which yields a principled inner-loop alternative to heuristic guidance-strength tuning.
    \item We empirically validate the conditional locality of video flow-matching latents, justifying a block-conditional Gibbs sampler that adaptively freezes locally-converged patches to accelerate convergence.
    \item On RealEstate10K and DAVIS, \textit{GibbsLens} attains the best FVD against all seven competitors and dominates training-free baselines across reported metrics.
\end{enumerate}

\section{Preliminaries}
\label{sec:prelim}

\paragraph{Latent flow matching.}
We work with pretrained latent flow-matching video generators
sampling in the VAE latent space $z\!=\!\mathcal E(x)$. Under the
variance-exploding reparameterization of the probability-flow
ODE~\citep{song2021scorebased,lipman2022flow,liu2022flow}, the
marginals are $z_t\!=\!(1-\sigma_t)\,z_0+\sigma_t\,\epsilon$ with
$\epsilon\!\sim\!\mathcal N(0,I)$ along the schedule
$\sigma_0\!=\!0,\sigma_1\!=\!1$. The learned velocity
$u_\theta(z_t,t,c)$, regressed on $\epsilon-z_0$, implements the
deterministic ODE $dz_t/dt\!=\!u_\theta(z_t,t,c)$ with $c$ the
external condition (e.g., a text prompt for T2V). Under the
velocity--score identity
$u_\theta\!=\!(\dot\sigma_t/\sigma_t)\,z_t-\dot\sigma_t\sigma_t\,\nabla_{z_t}\log p_t(z_t\!\mid\!c)$,
score-based reasoning transfers to flow matching. The plug-in
clean-prediction estimator
\begin{equation}
\label{eq:fm-cleanpred}
    \hat z_0(z_t,t,c) \;=\; z_t - \sigma_t\,u_\theta(z_t,t,c)
\end{equation}
is the primary interface to the guidance objectives in
Section~\ref{sec:method}.

\paragraph{Doob's $h$-transform for conditional sampling.}
Many tasks require conditioning the reverse process on an endpoint
event, e.g., a noisy inverse problem $y\!=\!\mathcal A(z_0)+n$.
Doob's $h$-transform~\citep{rogers2000diffusions} is the general
framework, tilting the path measure by a non-negative terminal
weight $h(z_0)$ (the likelihood $p(y\!\mid\!z_0)$ for soft
conditioning). With time-dependent
$h_t(z_t)\!=\!\mathbb E[h(z_0)\!\mid\!z_t]$~\citep{denker2024deft,wang2026coarse},
the conditioned reverse SDE adds an $h$-induced drift,
\begin{equation}
\label{eq:h-conditioned-reverse-sde}
    dz_t = \bigl[\,f(z_t,t) - g(t)^2\bigl(\nabla_{z_t}\log p_t(z_t)+\nabla_{z_t}\log h_t(z_t)\bigr)\,\bigr]\,dt + g(t)\,d\bar w_t,
\end{equation}
with drift $f$, diffusion $g$, and reverse-time Wiener process
$\bar w_t$; via the velocity--score identity, this gives the
controlled flow-matching velocity
\begin{equation}
\label{eq:h-controlled-velocity}
    u_\theta^{\rm ctrl}(z_t,t,c) \;=\; u_\theta(z_t,t,c) - \sigma_t\,\dot\sigma_t\,\nabla_{z_t}\log h_t(z_t).
\end{equation}
The $h$-induced drift is intractable in general; practical samplers
replace it with a tractable
surrogate~\citep{chung2022diffusion,wang2026coarse,wu2023practical},
e.g., the first-order Tweedie surrogate
$\nabla_{z_t}\log h_t(z_t)\!\approx\!\nabla_{z_t}\log h(\hat z_0(z_t,t,c))$
recovering DPS~\citep{chung2022diffusion}.
Appendix~\ref{sec:appendix-surrogates} reviews surrogate families
and locates \textit{GibbsLens}.

\section{Method}
\label{sec:method}

We treat training-free camera-controlled video generation as a
partial-observation inverse problem \cite{chung2022diffusion, lugmayr2022repaint}.
Given a target camera trajectory and a source video or first frame, we follow standard camera-control practice
\citep{singer2025time}: a monocular depth estimator lifts the source to a
point cloud, which is reprojected under the target poses to produce a
warped guidance video $\tilde x_0$ together with its pixel-space
visibility map. Encoding $\tilde x_0$ through the pretrained VAE and
downsampling the visibility map to the latent resolution yields the
warped latent $\tilde z_0=\mathcal E(\tilde x_0)$ and a binary mask
$M\in\{0,1\}^{L\times H\times W}$ marking the geometrically observed
sites on the latent grid of temporal length $L$ and spatial size
$H\!\times\!W$. Both signals are noisy because they
inherit depth-estimation errors, boundary
artifacts, and $M$ is only approximate due to the downsampling.

\paragraph{Vanilla $h$-transform guidance.}
We frame camera control as posterior sampling under a partial Gaussian
observation model. With the pretrained video generator as the prior
$p(z_0)$ and the warped latent $\tilde z_0$ as noisy evidence on the
masked sites,
\begin{equation}
    p(z_0\mid\tilde z_0,M)\;\propto\;p(z_0)\,\ell(\tilde z_0\mid z_0,M),
    \qquad
    \ell(\tilde z_0\mid z_0,M)\;\propto\;\exp\!\Bigl(-\tfrac{1}{2\tau^2}\|M\odot(z_0-\tilde z_0)\|_2^2\Bigr),
\end{equation}
with $\tau^{-2}$ the per-site observation confidence. This is the
Doob-transformed reverse process of
Eq.~\eqref{eq:h-conditioned-reverse-sde} with terminal weight
$h(z_0)=\ell(\tilde z_0\mid z_0,M)$. Following DPS
\citep{chung2022diffusion}, we approximate the intractable $h$-induced
drift by the first-order Tweedie surrogate
$\nabla_{z_t}\log h_t(z_t)\approx\nabla_{z_t}\log h(\hat z_0(z_t,t,c))$
and drop the denoiser Jacobian (Appendix~\ref{sec:appendix-surrogates}).
Substituting into Eq.~\eqref{eq:h-controlled-velocity} yields
\begin{equation}
\label{eq:m-vctrl}
    u_\theta^{\rm ctrl}(z_t,t,c)
    \;=\;
    u_\theta(z_t,t,c)
    \;+\;
    \frac{\sigma_t\dot\sigma_t}{\tau^2}\,M\odot\bigl(\hat z_0(z_t,t,c)-\tilde z_0\bigr),
\end{equation}
a soft pull toward $\tilde z_0$ on observed sites with strength
$\tau^{-2}$. As $\tau\!\to\!0$, the likelihood concentrates on
$\{z_0:M\odot z_0=M\odot\tilde z_0\}$ and
$\hat z_0^{\rm ctrl}=z_t-\sigma_t u_\theta^{\rm ctrl}$ collapses to
the hard masked replacement
\begin{equation}
\label{eq:m-obs-update}
    \hat z_0^{\rm obs}
    \;=\;
    M\odot\tilde z_0 \;+\; (1-M)\odot\hat z_0(z_t,t,c).
\end{equation}

We adopt Eq.~\eqref{eq:m-obs-update} as the outer guidance step of \textit{GibbsLens}:
the pretrained flow is preserved on $1\!-\!M$, and finite-confidence
behavior is restored by the inner refinement of
Section~\ref{sec:gibbs-refinement}.

\subsection{Conditional Pseudo-Gibbs Refinement}
\label{sec:gibbs-refinement}

The hard-replacement update Eq.~\eqref{eq:m-obs-update} is the
canonical training-free inpainting
rule~\citep{lugmayr2022repaint,singer2025time,song2025worldforge},
but two assumptions break here. The warped latent $\tilde z_0$
inherits depth-estimation errors and projection artifacts. And the unobserved support
$1\!-\!M$, left to the pretrained prior alone, must here follow the
target camera trajectory's spatio-temporal geometry rather than
merely look plausible. We address both by coupling each outer step
with an inner refinement loop at the same noise level.

\textbf{Inner-loop refinement.}
At fixed $\sigma_t$, the flow-matching backbone acts as a denoiser:
$\hat z_0(z_t,t,c)$ is its reconstruction from a noisy input.
\citet{bengio2013generalized} show that iterating perturb-and-redenoise
defines a generalized denoising auto-encoder (DAE) Markov chain whose
stationary distribution converges to $p(z_0)$. We extend
this primitive to a partial-observation conditional setting by
(i)~holding the observed sites at a freshly drawn noised pin from the
OT-FM forward kernel at level $\sigma_t$,
\begin{equation}
\label{eq:m-noised-obs}
    \bar z_t \;=\; (1-\sigma_t)\,\tilde z_0 \;+\; \sigma_t\,\xi_{\rm obs},
    \qquad \xi_{\rm obs}\sim\mathcal N(0,I),
\end{equation}
and (ii)~restricting the perturb step to the unobserved support.
Initializing $\hat z_0^{(0)}\!=\!\hat z_0^{\rm obs}$ from
Eq.~\eqref{eq:m-obs-update}, the inner chain alternates
\begin{align}
\label{eq:m-perturb}
    z_t^{(j)}
    &\;=\;
    (1-M)\odot\bigl((1-\sigma_t)\,\hat z_0^{(j-1)}+\sigma_t\,\xi^{(j)}\bigr)
    \;+\;
    M\odot\bar z_t,
    && \xi^{(j)}\!\sim\!\mathcal N(0,I),\\
\label{eq:m-redenoise}
    \hat z_0^{(j)}
    &\;=\;
    z_t^{(j)} \;-\; \sigma_t\,u_\theta\bigl(z_t^{(j)},t,c\bigr),
    && j=1,\dots,J_{\max},
\end{align}
followed by a hard-mask write-back that composes the two branches
before one FlowMatch Euler step advances the latent. The write-back
consumes the \emph{Polyak-averaged Gibbs readout}
$\bar z_0\!=\!\tfrac{1}{J_{\max}}\sum_{j=1}^{J_{\max}}\hat z_0^{(j)}$
rather than the last iterate $\hat z_0^{(J_{\max})}$:
\begin{equation}
\label{eq:m-writeback}
    \hat z_0^{\rm final}
    \;=\;
    M\odot\hat z_0^{\rm obs} \;+\; (1-M)\odot\bar z_0.
\end{equation}
By~\citet{polyak1992acceleration} on iterate averaging, this trades the
per-sample posterior variance for an
$\mathcal O(\tau_{\rm int}/J_{\max})$ Monte Carlo variance --- exactly
what the FlowMatch Euler step, which is linear in
$\hat z_0^{\rm final}$, expects.
Pinning observed sites at $\bar z_t$ rather than at the outer-state
$z_t|_M$ keeps the inner-loop conditioning at the correct
noised-observation distribution and decouples the chain from
outer-state drift. With $\sigma_t$ and $\bar z_t|_M$ held fixed,
Eqs.~\eqref{eq:m-perturb}--\eqref{eq:m-redenoise} form a generalized
DAE chain on the sub-state $z|_{1-M}$ whose stationary distribution
admits a closed-form characterization.

\begin{proposition}[Stationary distribution of conditional pseudo-Gibbs]
\label{prop:gibbs-stationary}
Fix denoising step $t$ with noise level $\sigma_t > 0$, mask $M$,
warped latent $\tilde z_0$, and pin $\bar z_t$ from
Eq.~\eqref{eq:m-noised-obs}. Assume \emph{(A1)} the denoiser
$u_\theta(\cdot,t,c)$ is consistent with the level-$\sigma_t$ joint
conditional on $c$ (denoising-score-matching optimality), and
\emph{(A2)} the perturbation kernel restricted to $1\!-\!M$ has full
support on a bounded volume. Then
Eqs.~\eqref{eq:m-perturb}--\eqref{eq:m-redenoise} define an ergodic
Markov chain whose clean-prediction iterates
$\hat z_0^{(j)}|_{1-M}$ converge in distribution, as
$j\!\to\!\infty$, to the partial-observation conditional data law
$p\!\bigl(z_0|_{1-M}\,\big|\,\bar z_t|_M\bigr)$.
\end{proposition}

The proof proceeds by reduction to the sub-state version of
\citet{bengio2013generalized} Theorem~1; it is given in
Appendix~\ref{sec:appendix-complement-theory}, alongside the
comparison with RePaint~\citep{lugmayr2022repaint} and
SRVS~\citep{jang2026self}.

\begin{figure}[t]
  \centering
  \begin{minipage}{0.8\linewidth}
    \centering
    \includegraphics[width=\linewidth]{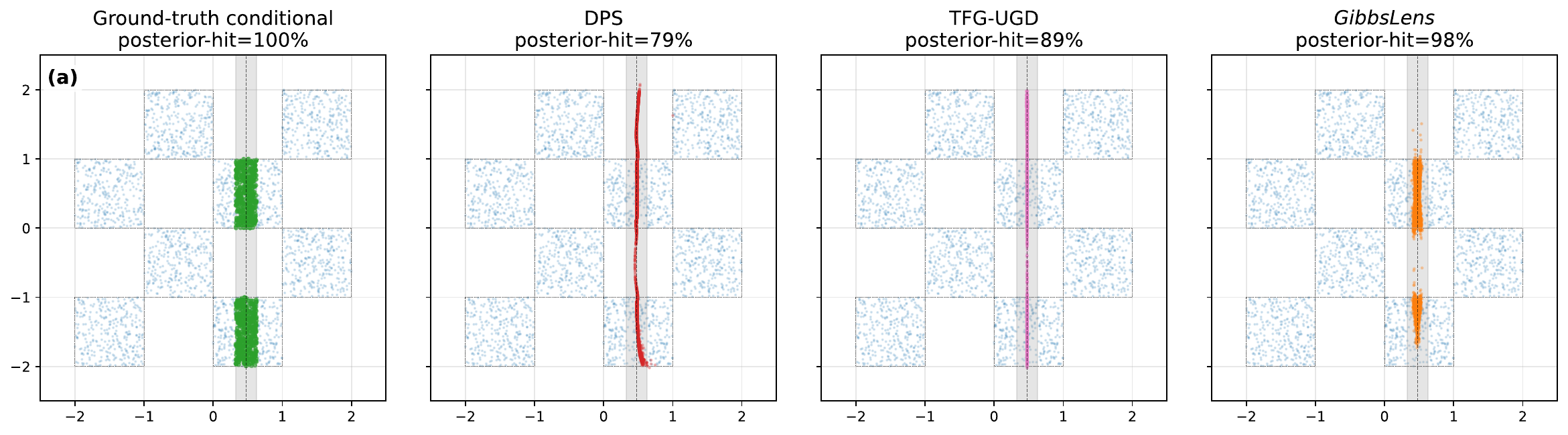}%
    \phantomsubcaption\label{fig:toy-a}\\[-0.6ex]
    \begin{minipage}[t]{0.5\linewidth}
      \centering
      \includegraphics[width=\linewidth]{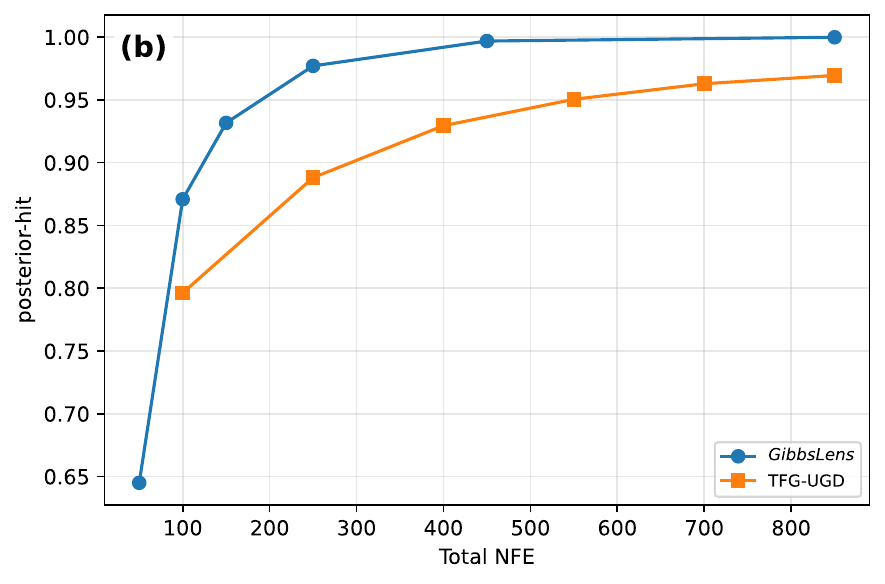}%
      \phantomsubcaption\label{fig:toy-b}
    \end{minipage}%
    \begin{minipage}[t]{0.5\linewidth}
      \centering
      \includegraphics[width=\linewidth]{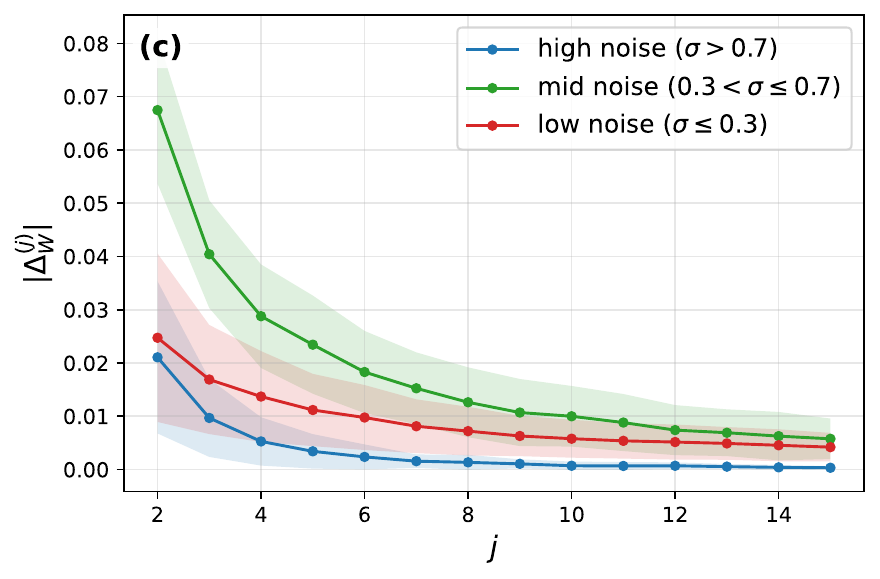}%
      \phantomsubcaption\label{fig:toy-c}
    \end{minipage}
  \end{minipage}
  \caption{2D checkerboard toy example.
  \textbf{(a)}~Sample clouds at $y_{\rm obs}\!\approx\!0.5$ for
  ground truth, DPS~\citep{chung2022diffusion},
  TFG-UGD~\citep{ye2024tfg}, and \textit{GibbsLens} (left to right).
  \textbf{(b)}~Posterior-hit rate vs.\ total NFE for \textit{GibbsLens}
  (varying $J_{max}$) and TFG-UGD (varying $N_{\rm recur}$).
  \textbf{(c)}~$|\Delta_W^{(j)}|$ vs.\ inner iteration $j$ binned
  by noise band.}
  \label{fig:toy}
\end{figure}

\textbf{Mixing detection.}
We monitor convergence through the trajectory of $\hat z_0^{(j)}$
itself. Let $\sigma_W^{(j)}$ denote the running standard deviation of
the inner iterates on the unobserved support, maintained by Welford's
online algorithm~\citep{welford1962note}. The $\Delta$-Welford
indicator is its first difference,
\begin{equation}
    \Delta_W^{(j)} \;=\; \sigma_W^{(j)}\,-\,\sigma_W^{(j-1)},
\end{equation}
with $|\Delta_W^{(j)}|\!\to\!0$ signalling that the chain has stopped
exploring.

\textbf{2D toy validation.}
We validate the construction on a 2D 8-square checkerboard with a
noisy partial observation on the first coordinate, anchored at
$y_{\rm obs}\!\approx\!0.5$ where the vertical constraint line yields a 2-mode conditional
posterior. The toy experiment establishes three claims:
\textit{(i)~Mode coverage.} As in Fig.~\ref{fig:toy-a},  
DPS~\citep{chung2022diffusion} and TFG-UGD~\citep{ye2024tfg} spread
mass along the constraint line without locking onto either mode, while \textit{GibbsLens} concentrates on both
conditional modes.
Disabling the inner loop reduces \textit{GibbsLens} to DPS-like behaviour,
isolating the refinement as the source of the gain.
\textit{(ii)~Compute scaling.} The posterior-hit rate of \textit{GibbsLens}
rises monotonically with the per-step probe count and strictly
dominates TFG-UGD at every matched NFE budget (Fig.~\ref{fig:toy-b}):
compute spent inside the unobserved sub-state pays off more than
TFG-style global recurrence.
\textit{(iii)~Indicator fidelity.} $|\Delta_W^{(j)}|$ decays over the
same $j$-range in which the posterior-hit curve of (b) plateaus,
across every noise band (Fig.~\ref{fig:toy-c}); the indicator is
therefore a reliable ground-truth-free trigger for the early-freeze
gate of Section~\ref{sec:method-locality-gibbs}. Detailed toy example implementation is in Appendix~\ref{sec:appendix-toy}.

\subsection{From Toy to Video: Locality and Block-Conditional Gibbs}
\label{sec:method-locality-gibbs}

At video scale --- the Wan~2.2 latent has shape $(C,L,H,W)$ with
$C\!=\!48$ and $\sim\!10^5$ sites --- a generalized DAE chain on
this full sub-state mixes too slowly (per-probe sampling variance
scales with state dimension). We therefore propose a novel \emph{block-conditional Gibbs} sampler over 3D patches, paired with
a per-patch $\Delta$-Welford freeze gate
(Section~\ref{sec:gibbs-refinement}); block updates dominate
single-site Gibbs in autocorrelation
reduction~\citep{liu1994covariance,roberts1997updating}. Such a
partition is legitimate only if the latent's conditional dependence
decays along each axis on a short range --- which we verify
empirically next.

\textbf{Latents are local.} Let $z^{(n)}_{l,h,w}\!\in\!\mathbb{R}^C$ denote the Wan~2.2 latent
token of video $n$ at grid position $(l,h,w)$. The diagnostic asks a
single question, one axis at a time: pick $\alpha\!\in\!\{L,H,W\}$ and
two positions $\beta,\gamma$ on $\alpha$; are the tokens at $\beta$
and $\gamma$ conditionally independent given the tokens at all other
positions on that line? The natural unit-free statistic for
multichannel tokens is the top \emph{block partial correlation}
$\rho_1(\widehat R_{\beta\gamma})\!\in\![0,1]$, the strongest linear
dependence between $z^{(n)}_\beta$ and $z^{(n)}_\gamma$ \emph{after
regressing out the rest of the line}; under joint Gaussianity,
$\rho_1\!=\!0$ is the exact conditional-independence test. See
Appendix~\ref{sec:appendix-locality-validation} for the detailed
implementation of the block partial correlation.

Figure~\ref{fig:heatmap} shows the off-diagonal
$\rho_1(\widehat R_{\beta\gamma})$ decays to a finite-sample noise
floor of $\le\!0.1$ within $|\beta-\gamma|\!\le\!2$ on each of
$L,H,W$: the latent is approximately order-$2$ Markov. This is conditional independence, not
separability --- the joint $p(z_0)$ stays non-factorisable, with
distant coordinates marginally correlated via chains of neighbours
--- and is consistent with the bounded receptive fields of
convolutional VAE
encoders~\citep{luo2016understanding,rombach2022high,esser2021taming,wan2025wan}.
The same locality is preserved on the model's clean prediction
$\hat z_0(z_t,t,c)$ across $\sigma_t\!\in\!\{0.1,0.3,0.5,0.7,0.9\}$
(Appendix~\ref{sec:appendix-locality-validation}).

\begin{figure}[!tb]
  \centering
  \includegraphics[width=0.8\linewidth]{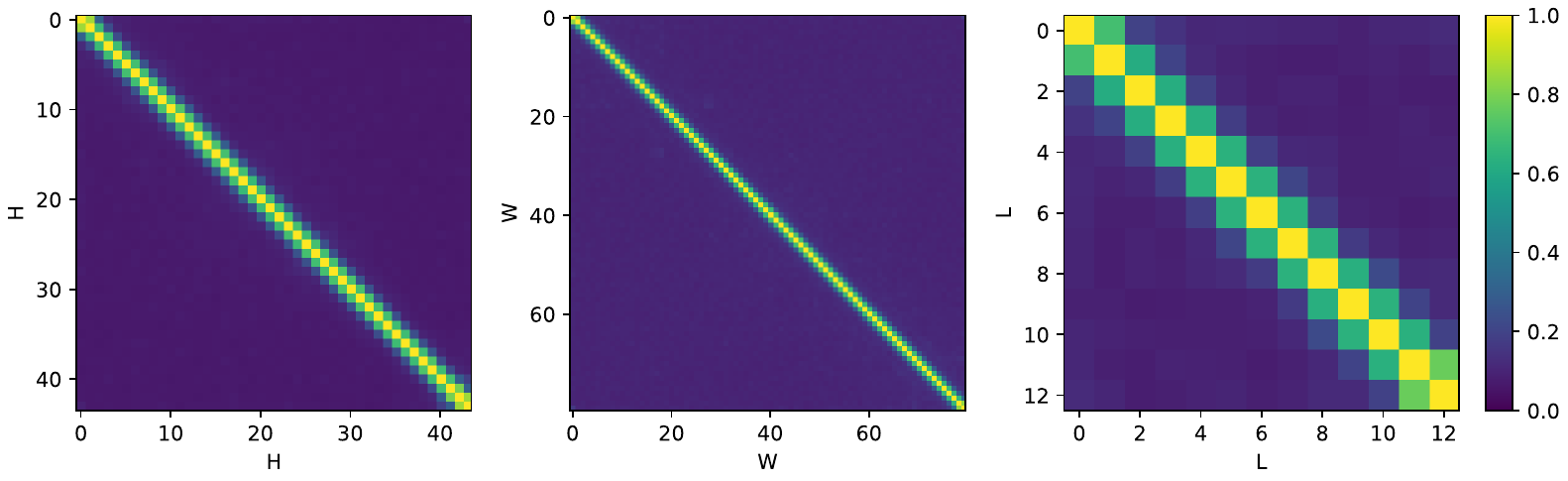}
  \caption{Top canonical partial correlation
  $\rho_1(\widehat R_{\beta\gamma})$ along the $H$, $W$, $L$ axes of
  the Wan~2.2 latent ($N\!=\!200$ encoded videos). Off-diagonal mass
  concentrates within $|\beta-\gamma|\!\le\!2$.}
  \label{fig:heatmap}
\end{figure}

\textbf{Patch-freeze block-conditional Gibbs.}
The locality structure justifies partitioning the unobserved support
$1\!-\!M$ into 3D patches $\{\mathcal P_g\}_{g=1}^{G}$ of size
$p_l\!\times\!p_h\!\times\!p_w$ chosen larger than the order-$2$
dependence radius. Locality then constrains
$p(z_0|_{\mathcal P_g}\!\mid\!z_0|_{\bar{\mathcal P}_g})$ to depend
only on a halo around $\mathcal P_g$, bounding the block-Gibbs
autocorrelation and allowing each patch to be frozen independently
once its inner chain stabilizes. Within each patch the inner-loop
perturb-and-redenoise of Eqs.~\eqref{eq:m-perturb}--\eqref{eq:m-redenoise}
is a block Gibbs update on the sub-state $z|_{\mathcal P_g\cap(1-M)}$
conditioned on a thin halo of neighboring patches plus the pinned
observed sites. We index the $\Delta$-Welford indicator of
Section~\ref{sec:gibbs-refinement} by patch as $\Delta_W^{(j)}(g)$
and track its running peak
$\Delta_W^{\max}(g)\!=\!\max_{j'\le j}|\Delta_W^{(j')}(g)|$. We declare a
patch \emph{stable} when its $|\Delta_W(g)|$ stays below a fraction
$\kappa\!\in\!(0,1]$ of $\Delta_W^{\max}(g)$ in two consecutive iterations:
\begingroup
\setlength{\abovedisplayskip}{3pt}%
\setlength{\belowdisplayskip}{3pt}%
\setlength{\abovedisplayshortskip}{3pt}%
\setlength{\belowdisplayshortskip}{3pt}%
\setlength{\jot}{1pt}%
\begin{equation}
\label{eq:m-stable-mask}
    S^{(j)}_g \;=\; \mathbf 1\!\bigl[\,|\Delta_W^{(j)}(g)|<\kappa \Delta_W^{\max}(g)\bigr]\,\wedge\,\mathbf 1\!\bigl[\,|\Delta_W^{(j-1)}(g)|<\kappa \Delta_W^{\max}(g)\bigr].
\end{equation}
Stable patches hold their previous clean-prediction iterate while
the rest update as usual:
\begin{equation}
\label{eq:m-freeze}
    \hat z_0^{(j)}\big|_{\mathcal P_g} \;\leftarrow\; S^{(j)}_g\,\hat z_0^{(j-1)}\big|_{\mathcal P_g} \;+\; (1-S^{(j)}_g)\,\hat z_0^{(j)}\big|_{\mathcal P_g},\quad g=1,\dots,G.
\end{equation}
\endgroup
The loop exits early when the fraction of stable patches
$\sum_g S^{(j)}_g/G$ exceeds $\nu$, and is capped at $J_{\max}$.

Algorithm~\ref{alg:hcontrol} gives the full \textit{GibbsLens} sampler
over the discretized schedule $\{t_k\}_{k=0}^{K}$
($\sigma_k\!=\!\sigma_{t_k}$, $z_k\!=\!z_{t_k}$) with camera-control
window $[s,e)\!\subseteq\![0,K)$ and at most $J_{\max}$ inner iterations
per active outer step, where $K$ denotes the number of discrete sampling
steps. Proposition~\ref{prop:gibbs-stationary} characterizes the asymptotic
limit of the underlying inner chain under its stated assumptions; it does
not provide a finite-step convergence guarantee for this implementation
with adaptive patch freezing, early stopping, and a Polyak-averaged readout.

\begin{algorithm}[t]
\caption{\textit{GibbsLens}: training-free camera-controlled flow-matching sampler.}
\label{alg:hcontrol}
\begin{algorithmic}[1]
\REQUIRE pretrained $u_\theta$, conditioning $c$; warped $\tilde z_0$, mask $M$; window $[s,e)$; $J_{\max}$; $\kappa,\nu\!\in\!(0,1]$.
\FOR{$k=0,\dots,K-1$}
  \STATE $\hat z_0\leftarrow z_k-\sigma_k u_\theta(z_k,t_k,c)$
  \IF{$s\le k<e$}
    \STATE $\hat z_0^{\rm obs}\leftarrow M\odot\tilde z_0+(1-M)\odot\hat z_0$
           \hfill Eq.~\eqref{eq:m-obs-update}
    \STATE $\bar z_k\leftarrow(1-\sigma_k)\tilde z_0+\sigma_k\xi_{\rm obs}$;\;
           $\hat z_0^{(0)}\leftarrow\hat z_0^{\rm obs}$;\;
           $\bar z_0,\,\Delta_W^{(0)}(g),\,\Delta_W^{\max}(g),\,S^{(0)}_g\!\leftarrow\!0$
           \hfill Eq.~\eqref{eq:m-noised-obs}
    \FOR{$j=1,\dots,J_{\max}$}
      \STATE $z_k^{(j)}\leftarrow(1-M)\odot\bigl((1-\sigma_k)\hat z_0^{(j-1)}+\sigma_k\xi^{(j)}\bigr)+M\odot\bar z_k$
             \hfill Eq.~\eqref{eq:m-perturb}
      \STATE $\hat z_0^{(j)}\leftarrow z_k^{(j)}-\sigma_k u_\theta(z_k^{(j)},t_k,c)$
             \hfill Eq.~\eqref{eq:m-redenoise}
      \STATE Refresh running $\Delta_W^{(j)}(g), \Delta_W^{\max}(g)$ (Welford bookkeeping)
      \STATE $S^{(j)}_g\!\leftarrow\!\mathbf 1\!\bigl[|\Delta_W^{(j)}(g)|\!<\!\kappa \Delta_W^{\max}(g)\bigr]\!\wedge\!\mathbf 1\!\bigl[|\Delta_W^{(j-1)}(g)|\!<\!\kappa \Delta_W^{\max}(g)\bigr]$
             \hfill Eq.~\eqref{eq:m-stable-mask}
      \STATE $\hat z_0^{(j)}|_{\mathcal P_g}\!\leftarrow\!S^{(j)}_g\hat z_0^{(j-1)}|_{\mathcal P_g}\!+\!(1\!-\!S^{(j)}_g)\hat z_0^{(j)}|_{\mathcal P_g}$
             \hfill Eq.~\eqref{eq:m-freeze}
      \STATE $\bar z_0\leftarrow\bar z_0+(\hat z_0^{(j)}-\bar z_0)/j$
             \hfill Polyak-averaged readout
      \STATE \textbf{break if} $\sum_g S^{(j)}_g/G > \nu$
             \hfill early exit on stable fraction
    \ENDFOR
    \STATE $\hat z_0^{\rm final}\leftarrow M\odot\hat z_0^{\rm obs}+(1-M)\odot\bar z_0$
           \hfill Eq.~\eqref{eq:m-writeback}
  \ELSE
    \STATE $\hat z_0^{\rm final}\leftarrow\hat z_0$
  \ENDIF
  \STATE $z_{k+1}\leftarrow z_k+(\sigma_{k+1}-\sigma_k)(z_k-\hat z_0^{\rm final})/\sigma_k$
\ENDFOR
\end{algorithmic}
\end{algorithm}

\section{Experiments}
\label{sec:exp}

\begingroup
\setlength{\textfloatsep}{4pt plus 1pt minus 1pt}
\setlength{\intextsep}{4pt plus 1pt minus 1pt}
\setlength{\floatsep}{4pt plus 1pt minus 1pt}
\setlength{\abovecaptionskip}{2pt}
\setlength{\belowcaptionskip}{0pt}
\renewcommand{\arraystretch}{0.9}


\textbf{Setup.}
We evaluate on \textit{RealEstate10K} (200 randomly sampled test
scenes) and \textit{DAVIS} (84 videos obtained from three
controlled trajectories of varying angular extent on each of 28
dynamic scenes), comparing \textit{GibbsLens} against four training-free
(WorldForge~\citep{song2025worldforge},
TTM~\citep{singer2025time},
RePaint~\citep{lugmayr2022repaint},
Coarse-Guided~\citep{wang2026coarse})
and three training-based ($\dagger$;
TrajectoryAttention~\citep{xiao2024trajectory},
TrajectoryCrafter~\citep{yu2025trajectorycrafter},
ReCamMaster~\citep{bai2025recammaster})
camera controllers. \textit{GibbsLens} and every training-free baseline
share the public Wan~2.2~\citep{wan2025wan}
(Text-Image-to-Video-5B) flow-matching backbone, ensuring a
controlled comparison in which the sampler is the only varied
factor; the warped guidance video and its visibility mask are
produced by the TTM/WorldForge depth-based lift-and-reproject
pipeline, while ReCamMaster takes its native camera-trajectory
input. All experiments run on
NVIDIA~A40 GPUs in mixed precision.

\textbf{Evaluation metrics.}
For visual quality, we retain Fr\'echet Video Distance
(FVD)~\citep{unterthiner2018towards} for comparability, while acknowledging
its limited sensitivity to temporal distortions~\citep{ge2024content}.
On both benchmarks, we complement it with Fr\'echet Video Motion
Distance (FVMD)~\citep{liu2024fvmd} for motion consistency and
VisionReward's overall score (VR-Overall)~\citep{xu2026visionreward}
as a learned proxy for human visual preferences. We also report
frame-wise LPIPS and SSIM against paired held-out reference views on
RealEstate10K, and
CLIP-v~\citep{kuang2024collaborative} and
CLIP-f~\citep{radford2021learning} on both benchmarks. For
trajectory adherence, camera poses re-extracted with
Mega-SAM~\citep{li2025megasam} yield the Absolute Trajectory Error
(ATE), Relative Rotation Error (RRE), and Relative Translation
Error (RTE).

\subsection{Comparison with State-of-the-Art Methods}
\label{sec:exp-sota}

\textbf{Static scenes (RealEstate10K).}
Table~\ref{tab:re10k-quant} reports quantitative results.
\textit{GibbsLens} matches or beats every training-free baseline across all
ten metrics, reducing FVD from $157.08$ to $\mathbf{129.25}$
($17.7\%$) and FVMD from $1606.55$ to $1314.96$ ($18.2\%$)
relative to Coarse-Guided. It also achieves the highest VR-Overall
($\mathbf{0.1277}$ vs.\ $0.1153$ for TTM), alongside gains in
trajectory error, LPIPS, and CLIP-based scores, with SSIM tied at $0.62$.
Across all methods, \textit{GibbsLens} ranks first on seven of ten metrics
and second on FVMD. TrajectoryCrafter achieves lower FVMD ($1151.24$)
and RRE ($0.4235$), and TrajectoryAttention higher SSIM ($0.64$),
but their FVDs are $2$--$5\!\times$ ours and their VR-Overall scores
are lower ($0.0525$ and $0.0200$, respectively). These results support
improvements over training-free baselines in motion consistency and
preference-aligned quality, beyond FVD and paired-frame fidelity.
Figure~\ref{fig:real10k_exp} shows qualitative comparisons: \textit{GibbsLens} synthesises photorealistic novel views
aligned with the target poses, while baselines exhibit visible
warp seams and disocclusion artifacts.
\begin{table}[h]
  \caption{Visual quality and trajectory-error results on
  \textbf{RealEstate10K}. Best across all methods in \textbf{bold},
  second-best \underline{underlined}. $^\dagger$ marks training-based
  methods.}
  \centering
  \footnotesize
  \setlength{\tabcolsep}{2pt}
  \renewcommand{\arraystretch}{0.7}
  \begin{tabular}{@{}lcccccccccc@{}}
  \toprule
                & FVD$\downarrow$ & FVMD$\downarrow$ & VR-Overall$\uparrow$ & \shortstack{ATE\\(cm)$\downarrow$} & \shortstack{RRE\\(deg)$\downarrow$} & \shortstack{RTE\\(cm)$\downarrow$} & LPIPS$\downarrow$ & SSIM$\uparrow$ & CLIP-f$\uparrow$ & CLIP-v$\uparrow$ \\
  \midrule
  WorldForge~\citep{song2025worldforge}                    & 158.47 & 1915.55 & 0.0968 & 5.8838  & 0.4373 & 1.4176 & \underline{0.42} & 0.56 & 97.87 & 92.54 \\
  TTM~\citep{singer2025time}                           & 192.88 & 1611.57 & \underline{0.1153} & 3.8700  & 0.4751 & 1.1745 & 0.48 & 0.52 & 98.16 & \underline{93.39} \\
  Coarse-Guided~\citep{wang2026coarse}                 & \underline{157.08} & 1606.55 & 0.0737 & \underline{3.5719} & 0.4403 & \underline{0.9535} & \underline{0.42} & 0.62 & 97.84 & 91.60 \\
  RePaint~\citep{lugmayr2022repaint}                       & 265.00 & 6515.38 & 0.0812 & 10.3861 & 0.7438 & 2.8434 & 0.52 & 0.54 & 97.66 & 90.13 \\
  \midrule
  TrajectoryAttention$^\dagger$~\citep{xiao2024trajectory} & 616.81 & 9858.78 & 0.0200 & 18.0606 & 0.8492 & 3.2817 & 0.45 & \textbf{0.64} & 97.10 & 87.23 \\
  TrajectoryCrafter$^\dagger$~\citep{yu2025trajectorycrafter}   & 253.27 & \textbf{1151.24} & 0.0525 & 4.1568  & \textbf{0.4235} & 0.9794 & 0.43 & \underline{0.63} & 97.71 & 91.50 \\
  ReCamMaster$^\dagger$~\citep{bai2025recammaster}         & 299.33 & 11554.00 & 0.0111 & 5.1924  & 0.4627 & 1.6436 & 0.59 & 0.53 & \underline{98.21} & 89.34 \\
  \midrule
  \textbf{\textit{GibbsLens} (Ours)}   & \textbf{129.25} & \underline{1314.96} & \textbf{0.1277} & \textbf{3.4920} & \underline{0.4321} & \textbf{0.9483} & \textbf{0.41} & 0.62 & \textbf{98.52} & \textbf{93.80} \\
  \bottomrule
  \end{tabular}
  \label{tab:re10k-quant}
  \end{table}

\begin{figure}[t]
  \centering
  \includegraphics[width=\dimexpr\linewidth-2\fboxsep-2\fboxrule\relax]{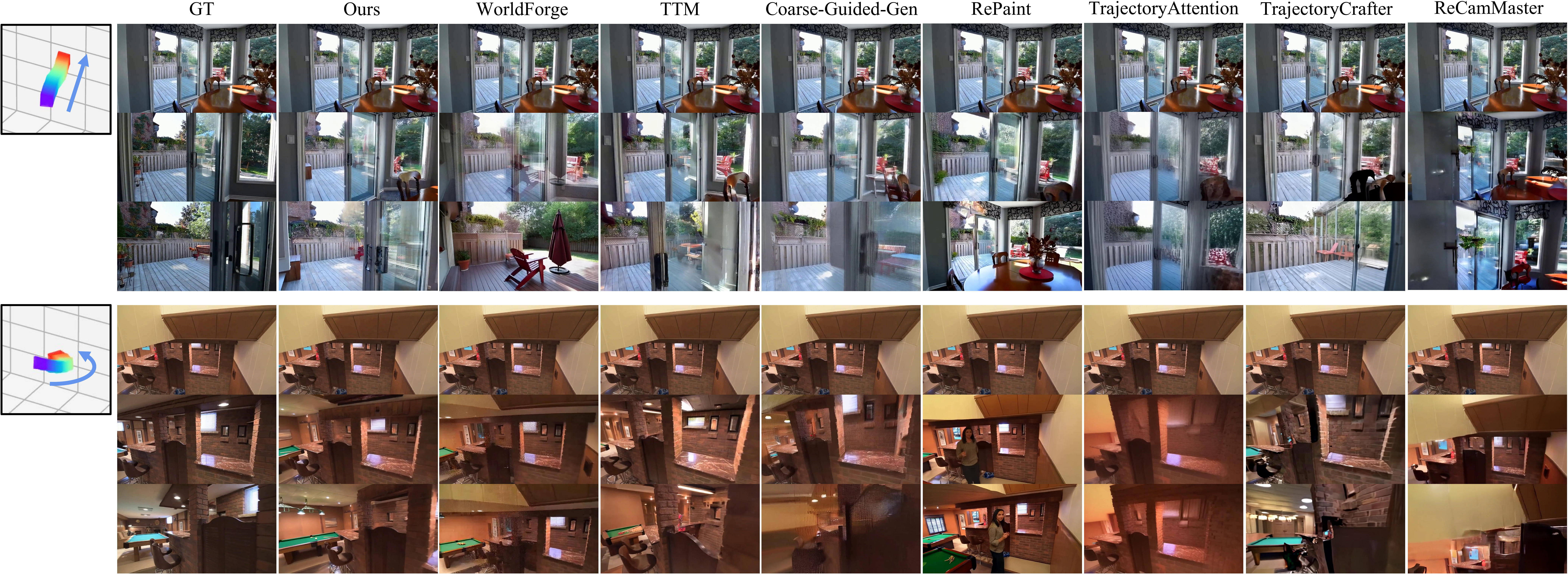}
  \caption{Qualitative results on RealEstate10K. Compared with the baselines, our method
  generates photorealistic novel views that align with the
  target camera poses.}
  \label{fig:real10k_exp}
\end{figure}

\textbf{Dynamic scenes (DAVIS).}
Quantitative results are reported in Table~\ref{tab:davis-quant}.
Against the four training-free baselines, \textit{GibbsLens} is again
best on every metric. Relative to the strongest baseline for each
quality metric, it reduces FVD from $383.53$ to $\mathbf{356.73}$
and FVMD from $9288.76$ to $\mathbf{8782.67}$ ($5.4\%$),
both against WorldForge, and increases VR-Overall from $0.0918$
to $\mathbf{0.0974}$ against TTM. These gains provide complementary
evidence of improved motion consistency and preference-aligned quality
in dynamic scenes. The comparison with training-based
controllers exposes a clear quality-versus-control trade-off:
ReCamMaster, fine-tuned end-to-end as a pose-conditioned
generator, achieves the lowest trajectory errors in terms of ATE, RRE, and RTE.
This trajectory-metric lead is structural: only ReCamMaster
conditions directly on the camera trajectory, whereas all others operate on a depth-warped guidance video whose
projection errors inherently cap trajectory accuracy.
However, its FVD ($670.42$), FVMD ($25387.40$), and VR-Overall
($-0.0406$) indicate weaker visual quality than \textit{GibbsLens}.
\textit{GibbsLens} instead achieves the best FVD, FVMD, VR-Overall,
CLIP-v, and CLIP-f overall, together with the
second-best trajectory errors after ReCamMaster, attaining the
most favourable balance between trajectory adherence and visual
quality among the seven controllers.
Figure~\ref{fig:davis-qualitative} visualises the same trade-off:
hard-replacement baselines show seams and unstable disoccluded
content, ReCamMaster's appearance degrades under dynamic content,
while \textit{GibbsLens} remains seam-free and geometrically faithful.

\begin{table}[t]
\caption{Trajectory-error and visual quality results on \textbf{DAVIS}. Best
across all methods in \textbf{bold}, second-best \underline{underlined}.
$^\dagger$ marks training-based methods.}
\centering
\footnotesize
\setlength{\tabcolsep}{3pt}
\renewcommand{\arraystretch}{0.7}
\begin{tabular}{@{}lcccccccc@{}}
\toprule
& FVD$\downarrow$ & FVMD$\downarrow$ & VR-Overall$\uparrow$ & ATE\,(cm)$\downarrow$ & RRE\,(deg)$\downarrow$ & RTE\,(cm)$\downarrow$ & CLIP-f$\uparrow$ & CLIP-v$\uparrow$ \\
\midrule
WorldForge~\citep{song2025worldforge} & \underline{$383.53$} & \underline{$9288.76$} & $0.0658$ & $8.2694$  & $1.9323$ & $5.1275$ & $97.05$ & $89.25$\\
TTM~\citep{singer2025time} & $398.12$ & $11948.70$ & \underline{$0.0918$} & $7.8849$  & $1.9111$ & $4.3735$ & \underline{$97.35$} & \underline{$89.85$}\\
Coarse-Guided~\citep{wang2026coarse} & $402.14$ & $11414.45$ & $0.0116$ & $9.0672$  & $2.1216$ & $4.3978$ & $96.40$ & $87.35$\\
RePaint~\citep{lugmayr2022repaint} & $690.83$ & $10820.76$ & $-0.1147$ & $9.9780$  & $2.2468$ & $5.1842$ & $95.33$ & $74.76$\\
\midrule
TrajectoryAttention$^\dagger$~\citep{xiao2024trajectory} & $837.37$ & $16028.85$ & $-0.0816$ & $18.3749$ & $3.6871$ & $7.2034$ & $93.88$ & $74.23$\\
TrajectoryCrafter$^\dagger$~\citep{yu2025trajectorycrafter} & $576.30$ & $10246.31$ & $-0.0626$ & $7.3827$  & $2.1045$ & $4.7158$ & $95.86$ & $83.21$\\
ReCamMaster$^\dagger$~\citep{bai2025recammaster} & $670.42$ & $25387.40$ & $-0.0406$ & $\mathbf{4.9770}$ & $\mathbf{0.6352}$ & $\mathbf{1.7271}$ & $95.39$ & $85.06$\\
\midrule
\textbf{\textit{GibbsLens} (Ours)}   & $\mathbf{356.73}$ & $\mathbf{8782.67}$ & $\mathbf{0.0974}$ & \underline{$6.8832$} & \underline{$1.8907$} & \underline{$4.2285$} & $\mathbf{97.45}$ & $\mathbf{90.34}$ \\
\bottomrule
\end{tabular}
\label{tab:davis-quant}
\end{table}

\begin{figure}[t]
  \centering
  \includegraphics[width=\dimexpr\linewidth-2\fboxsep-2\fboxrule\relax]{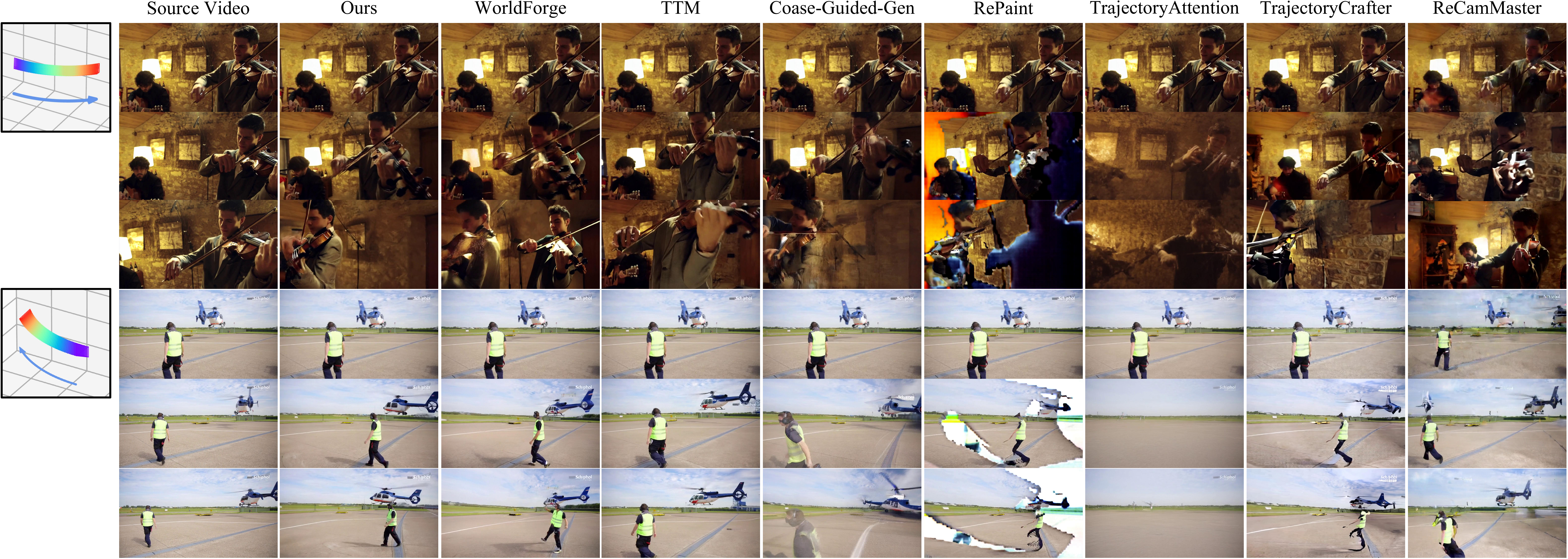}
  \caption{Qualitative results on DAVIS. On dynamic scenes,
  \textit{GibbsLens} achieves accurate camera control and good
  visual quality simultaneously, while baselines trade one for the other.}
  \label{fig:davis-qualitative}
\end{figure}

Across both regimes, \textit{GibbsLens} matches or exceeds every training-free
baseline on every reported metric and achieves the best FVD and
VR-Overall among all compared methods, with the best FVMD on DAVIS and second-best on
RealEstate10K. These complementary metrics support a favourable balance
of visual quality, motion consistency, and trajectory adherence without
fine-tuning.

\textbf{Computational cost.}
Table~\ref{tab:compute-cost} compares training-free methods using the
same Wan~2.2 backbone to generate a 49-frame video at
$720{\times}1280$ on one NVIDIA~A40 GPU. Relative to TTM,
\textit{GibbsLens} increases runtime from $227$ to $413$ seconds ($81.9\%$),
with nearly unchanged peak GPU memory ($26.1641$ vs.\ $26.1719$ GiB).
It remains faster than RePaint and WorldForge in this
setting. Its full-budget denoiser count is $48+3\!\times\!10=78$
evaluations; early stopping in Algorithm~\ref{alg:hcontrol} can reduce
this count. NFE alone does not determine wall-clock cost, since auxiliary
operations and implementation differences also contribute to runtime.

\begin{table}[!htbp]
\caption{Computational cost of training-free camera control under a
shared backbone and evaluation setting. NFE counts denoiser evaluations;
for \textit{GibbsLens}, the full-budget count is reported.}
\centering
\footnotesize
\setlength{\tabcolsep}{12pt}
\renewcommand{\arraystretch}{1.0}
\begin{tabular}{lrrr}
\toprule
Method & Time (s) & NFE & Peak GPU memory (GiB) \\
\midrule
WorldForge~\citep{song2025worldforge} & 751 & 65 & 37.0781 \\
TTM~\citep{singer2025time} & 227 & 48 & 26.1641 \\
Coarse-Guided~\citep{wang2026coarse} & 235 & 50 & 26.1289 \\
RePaint~\citep{lugmayr2022repaint} & 472 & 100 & 25.6328 \\
\midrule
\textbf{\textit{GibbsLens} (Ours)} & 413 & 78 & 26.1719 \\
\bottomrule
\end{tabular}
\label{tab:compute-cost}
\end{table}

\subsection{Analysis and Ablation Studies}
\label{sec:exp-ablation}

\begin{table}[h]
\caption{Component ablation on \textbf{DAVIS}. \checkmark marks
active components; top row is the outer $h$-transform guidance alone,
bottom row the full sampler. Best in \textbf{bold}, second-best
\underline{underlined}.}
\centering
\footnotesize
\setlength{\tabcolsep}{5pt}
\renewcommand{\arraystretch}{0.9}
\begin{tabular}{ccc cccccc}
\toprule
\shortstack{Pseudo-Gibbs\\refinement} & \shortstack{Polyak\\readout} & \shortstack{Patch-freeze\\gate} & FVD$\downarrow$ & ATE\,(cm)$\downarrow$ & RRE\,(deg)$\downarrow$ & RTE\,(cm)$\downarrow$ & CLIP-f$\uparrow$ & CLIP-v$\uparrow$\\
\midrule
            &            &            & 392.13 & 7.9558 & 2.0128 & 4.4130 & 96.63 & 88.93 \\
\checkmark  &            &            & 382.26 & 7.6343 & 1.9554 & \underline{4.3351} & \underline{96.98} & \underline{89.71} \\
\checkmark  & \checkmark &            & \underline{375.67} & \underline{7.6051} & 2.0758 & 4.5340 & 96.91 & 89.49 \\
\checkmark  &            & \checkmark & 382.26 & 7.8734 & \textbf{1.8775} & 4.5266 & 96.96 & 89.62\\
\midrule
\checkmark  & \checkmark & \checkmark & \textbf{356.73} & \textbf{6.8832} & \underline{1.8907} & \textbf{4.2285} & \textbf{97.45} & \textbf{90.34} \\
\bottomrule
\end{tabular}
\label{tab:ablation-component}
\end{table}

\textbf{Component ablation.}
Table~\ref{tab:ablation-component} isolates the three sampler
components introduced in Section~\ref{sec:method} ---
\emph{pseudo-Gibbs refinement} (Section~\ref{sec:gibbs-refinement}),
the \emph{Polyak-averaged readout}
(Eq.~\eqref{eq:m-writeback}), and the
\emph{$\Delta$-Welford patch-freeze gate}
(Section~\ref{sec:method-locality-gibbs}) --- by toggling each on
top of the outer $h$-transform step on DAVIS. Activating
\emph{pseudo-Gibbs refinement} alone (row 1 $\!\to\!$ 2) reduces FVD
from $392.13$ to $382.26$ and improves every other metric uniformly,
confirming that the perturb-and-redenoise primitive on the
unobserved support already provides a moderate uniform gain over
the outer $h$-transform step alone. Adding the \emph{Polyak-averaged readout}
(row 2 $\!\to\!$ 3) further drops FVD to $375.67$ with CLIP scores
essentially unchanged, consistent with the variance reduction
afforded by iterate averaging on the linear FlowMatch step; trajectory
errors drift slightly upward (RRE $2.08$, RTE $4.53$) because
averaging smooths the per-iterate signal that drives sharp pose
alignment. The \emph{$\Delta$-Welford patch-freeze gate}
(row 2 $\!\to\!$ 4) instead attains the best RRE in the table
($\mathbf{1.8775}$) while keeping CLIP scores stable; we
hypothesise that the gain stems from
fixing converged patches at their stabilised iterate before drift sets in on rotation-sensitive
regions. FVD, however, remains unchanged ($382.26$): the gate
protects already-mixed structure rather than improving
distributional fidelity on its own. The two enhancements are
therefore \emph{complementary}: combining all three (row 5) attains
the best score on five of six metrics and second-best on RRE, with
the freeze gate and the Polyak readout compensating for each other's
costs and yielding gains substantially larger than their individual
contributions.

\textbf{Block-conditional Gibbs mixing.}
Figure~\ref{fig:stable-mask} visualises the per-patch stability
mask $S_g$ across inner iterations on three frames. The top
two rows show $S_g$ at $j\!=\!4$ and $j\!=\!12$, and the bottom
row the corresponding warped frames. As $j$ grows, the stable
region (frozen patches) expands from a sparse seed and eventually
fills the unobserved support. This patch-by-patch convergence directly validates the latent locality
assumption: if patches were globally coupled, no patch could
stabilise before the rest, and the stability mask would never
exhibit a spatially structured growth pattern.

\begin{figure}[t]
  \centering
  \includegraphics[width=0.5\linewidth]{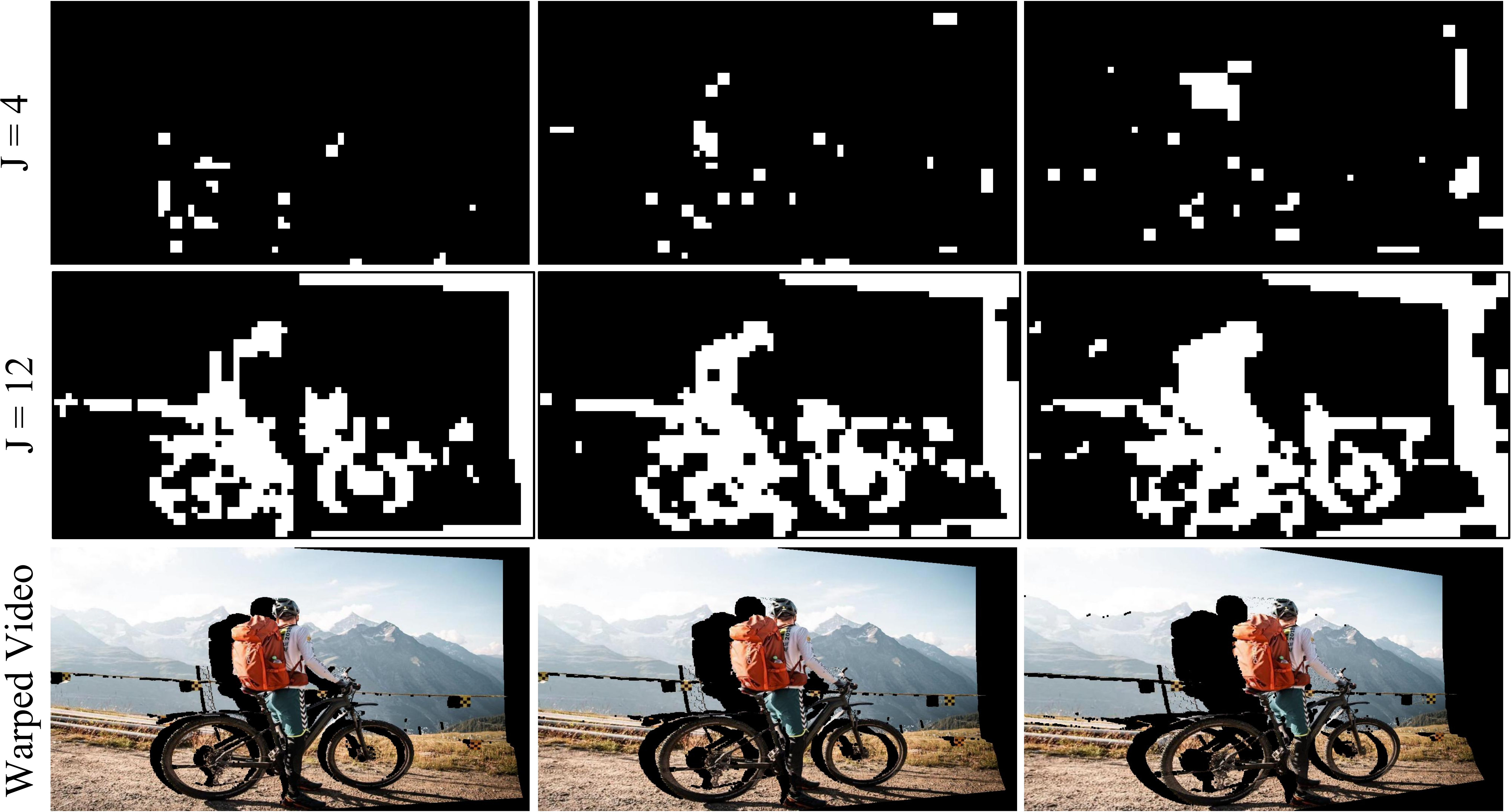}
  \caption{\textbf{Stability mask $S_g$ evolution} at the initial
  outer guidance step (noise level $\sigma_{t_s}$). Stable region
  grows with $j$ to cover the unobserved support.}
  \label{fig:stable-mask}
\end{figure}

\endgroup

\section{Related Works}
\label{sec:related}

\textbf{Camera-Controllable Video Generation.}
Existing camera-control methods can be broadly divided into
\emph{training-based} and \emph{training-free} paradigms. Training-based
methods inject camera trajectories, camera-aware representations, or
multi-view synchronization modules into the generator and learn the
mapping from control signals to motion through additional
supervision~\citep{he2024cameractrl,he2025cameractrl,bahmani2024vd3d,bahmani2025ac3d,wang2024motionctrl,bai2024syncammaster,bai2025recammaster,van2024generative};
they achieve strong controllability but rely on model-specific
retraining and suitable camera-annotated or multi-view data. By
contrast, training-free methods construct geometric guidance at
inference time, e.g., via point-cloud or depth-based warping, latent
reframing, or explicit view-conditioned rendering
cues~\citep{hou2024training,zhou2025latent,zhang2025recapture,park2025zero4d},
with related trajectory-editing and novel-view synthesis systems
combining explicit geometry with generative models for monocular video
re-rendering or sparse-view
synthesis~\citep{yu2024viewcrafter,yu2025trajectorycrafter,seo2025vid}.
Our method belongs to the training-free family but differs in how the
warped signal is used: instead of treating warped content as a
replacement target, we interpret it as \emph{partial and spatially
non-uniform evidence} and refine the unobserved complement through a
block-conditional pseudo-Gibbs sampler equipped with an adaptive
$\Delta$-Welford patch-freeze gate.

\textbf{Training-Free Conditional Sampling.}
Diffusion Posterior Sampling (DPS)~\citep{chung2022diffusion}
instantiates Doob's $h$-transform under a Gaussian observation
likelihood via a Tweedie surrogate; the noise-aware
weighted-$h$-transform of~\citet{wang2026coarse} extends DPS with a
global-scalar confidence weight,
and~\citet{zhu2026training} pursues the same formal object via
fine-tuning rather than inference-time guidance.
Inference-time conditional samplers such as the Twisted Diffusion
Sampler and Feynman--Kac
steering~\citep{wu2023practical,singhal2025general} steer pretrained
models without retraining, and MPGD~\citep{he2023manifold} shows that
aggressive guidance can push trajectories off the data manifold. Our
inner refinement extends Bengio's generalised denoising
auto-encoder~\citep{bengio2013generalized} to the partial-observation
conditional setting; SRVS~\citep{jang2026self} uses the same primitive
unconditionally on the full latent, and plug-and-play /
restoration-prior
methods~\citep{sreehari2016plug,romano2017little,hu2023restoration,terris2025fire}
share the spirit of using a pretrained denoiser as an implicit prior.
Detailed positioning and SRVS distinctions are in Appendix~\ref{sec:appendix-srvs-diffs}.

\section{Conclusion and Discussion}
\label{sec:conclusion}

\textit{GibbsLens} is a training-free posterior sampler for pretrained
flow-matching video generators that recasts camera control as
conditional posterior sampling under a partial Gaussian observation
likelihood. Its technical contribution is a
\emph{block-conditional pseudo-Gibbs} refinement on the unobserved
sub-state. We accelerate convergence on high-dimensional video latents
through a locality-justified patch partition, made adaptive by a
$\Delta$-Welford patch-freeze gate that holds each patch fixed
once its inner iterates have converged. Across RealEstate10K and DAVIS, \textit{GibbsLens} attains the
best FVD against all seven training-free and training-based
competitors and improves over every training-free baseline on
every metric, while remaining competitive with
training-based controllers without fine-tuning.

\textbf{Limitations.}
Two caveats are worth noting. First, the order-$2$ Markov
property used to size the block-conditional patches is calibrated
to the Wan~2.2 latent
(Appendix~\ref{sec:appendix-locality-validation}). Latent locality
itself is a generic property of VAE-encoded video models, but the precise dependence order must be
re-measured per backbone before \textit{GibbsLens} is applied. Second, the outer step
inherits the geometric accuracy of the depth-and-reproject warping
pipeline --- severe depth errors degrade the conditioning
likelihood, a constraint shared by every training-free method that
consumes warped guidance.

\textbf{Future directions.} Applying \textit{GibbsLens} to other partial-observation video inverse problems such as video inpainting, super-resolution and masked novel-view synthesis is a promising direction, since these tasks share the same partial-evidence structure and could be tackled by the same
training-free sampler. 
Another open challenge is to accelerate
the refinement process: the global attention still
incurs a full forward pass over every token regardless of which patches are frozen, so sparse-attention or KV-caching variants of DiT that skip frozen tokens, or distilled re-denoising surrogates,
are needed to translate the freeze gate's adaptivity into
wall-clock speed-ups.

\begin{ack}
This work was supported in part by the NSFC Special Fund and Major Program under
Grants 62550004 and 62595773; in part by the Fundamental and Interdisciplinary
Disciplines Breakthrough Plan of the Ministry of Education of China
(No.~JYB2025XDXM101).
\end{ack}

\bibliographystyle{unsrtnat}
\bibliography{main}

@String(AAAI  = {AAAI})

@article{bai2024syncammaster,
  title={Syncammaster: Synchronizing multi-camera video generation from diverse viewpoints},
  author={Bai, Jianhong and Xia, Menghan and Wang, Xintao and Yuan, Ziyang and Fu, Xiao and Liu, Zuozhu and Hu, Haoji and Wan, Pengfei and Zhang, Di},
  journal={arXiv preprint arXiv:2412.07760},
  year={2024}
}

@inproceedings{bai2025recammaster,
  title={Recammaster: Camera-controlled generative rendering from a single video},
  author={Bai, Jianhong and Xia, Menghan and Fu, Xiao and Wang, Xintao and Mu, Lianrui and Cao, Jinwen and Liu, Zuozhu and Hu, Haoji and Bai, Xiang and Wan, Pengfei and others},
  booktitle={Proceedings of the IEEE/CVF International Conference on Computer Vision},
  pages={14834--14844},
  year={2025}
}

@inproceedings{zhang2025recapture,
  title={Recapture: Generative video camera controls for user-provided videos using masked video fine-tuning},
  author={Zhang, David Junhao and Paiss, Roni and Zada, Shiran and Karnad, Nikhil and Jacobs, David E and Pritch, Yael and Mosseri, Inbar and Shou, Mike Zheng and Wadhwa, Neal and Ruiz, Nataniel},
  booktitle={Proceedings of the IEEE/CVF Conference on Computer Vision and Pattern Recognition},
  pages={2050--2062},
  year={2025}
}

@inproceedings{van2024generative,
  title={Generative camera dolly: Extreme monocular dynamic novel view synthesis},
  author={Van Hoorick, Basile and Wu, Rundi and Ozguroglu, Ege and Sargent, Kyle and Liu, Ruoshi and Tokmakov, Pavel and Dave, Achal and Zheng, Changxi and Vondrick, Carl},
  booktitle={European Conference on Computer Vision},
  pages={313--331},
  year={2024},
  organization={Springer}
}

@inproceedings{yu2025trajectorycrafter,
  title={Trajectorycrafter: Redirecting camera trajectory for monocular videos via diffusion models},
  author={Yu, Mark and Hu, Wenbo and Xing, Jinbo and Shan, Ying},
  booktitle={Proceedings of the IEEE/CVF international conference on computer vision},
  pages={100--111},
  year={2025}
}

@article{xiao2024trajectory,
  title={Trajectory attention for fine-grained video motion control},
  author={Xiao, Zeqi and Ouyang, Wenqi and Zhou, Yifan and Yang, Shuai and Yang, Lei and Si, Jianlou and Pan, Xingang},
  journal={arXiv preprint arXiv:2411.19324},
  year={2024}
}

@article{he2024cameractrl,
  title={Cameractrl: Enabling camera control for text-to-video generation},
  author={He, Hao and Xu, Yinghao and Guo, Yuwei and Wetzstein, Gordon and Dai, Bo and Li, Hongsheng and Yang, Ceyuan},
  journal={arXiv preprint arXiv:2404.02101},
  year={2024}
}

@inproceedings{he2025cameractrl,
  title={Cameractrl ii: Dynamic scene exploration via camera-controlled video diffusion models},
  author={He, Hao and Yang, Ceyuan and Lin, Shanchuan and Xu, Yinghao and Wei, Meng and Gui, Liangke and Zhao, Qi and Wetzstein, Gordon and Jiang, Lu and Li, Hongsheng},
  booktitle={Proceedings of the IEEE/CVF International Conference on Computer Vision},
  pages={13416--13426},
  year={2025}
}

@article{kuang2024collaborative,
  title={Collaborative video diffusion: Consistent multi-video generation with camera control},
  author={Kuang, Zhengfei and Cai, Shengqu and He, Hao and Xu, Yinghao and Li, Hongsheng and Guibas, Leonidas J and Wetzstein, Gordon},
  journal={Advances in Neural Information Processing Systems},
  volume={37},
  pages={16240--16271},
  year={2024}
}

@article{bahmani2024vd3d,
  title={Vd3d: Taming large video diffusion transformers for 3d camera control},
  author={Bahmani, Sherwin and Skorokhodov, Ivan and Siarohin, Aliaksandr and Menapace, Willi and Qian, Guocheng and Vasilkovsky, Michael and Lee, Hsin-Ying and Wang, Chaoyang and Zou, Jiaxu and Tagliasacchi, Andrea and others},
  journal={arXiv preprint arXiv:2407.12781},
  year={2024}
}

@inproceedings{bahmani2025ac3d,
  title={Ac3d: Analyzing and improving 3d camera control in video diffusion transformers},
  author={Bahmani, Sherwin and Skorokhodov, Ivan and Qian, Guocheng and Siarohin, Aliaksandr and Menapace, Willi and Tagliasacchi, Andrea and Lindell, David B and Tulyakov, Sergey},
  booktitle={Proceedings of the Computer Vision and Pattern Recognition Conference},
  pages={22875--22889},
  year={2025}
}

@inproceedings{wang2024motionctrl,
  title={Motionctrl: A unified and flexible motion controller for video generation},
  author={Wang, Zhouxia and Yuan, Ziyang and Wang, Xintao and Li, Yaowei and Chen, Tianshui and Xia, Menghan and Luo, Ping and Shan, Ying},
  booktitle={ACM SIGGRAPH 2024 Conference Papers},
  pages={1--11},
  year={2024}
}

@article{lipman2022flow,
  title={Flow matching for generative modeling},
  author={Lipman, Yaron and Chen, Ricky TQ and Ben-Hamu, Heli and Nickel, Maximilian and Le, Matt},
  journal={arXiv preprint arXiv:2210.02747},
  year={2022}
}

@book{rogers2000diffusions,
  title={Diffusions, Markov processes, and martingales},
  author={Rogers, L Chris G and Williams, David},
  volume={2},
  year={2000},
  publisher={Cambridge university press}
}

@article{zhu2026training,
  title={Training-Free Adaptation of Diffusion Models via Doob's $ h $-Transform},
  author={Zhu, Qijie and Ye, Zeqi and Liu, Han and Wang, Zhaoran and Chen, Minshuo},
  journal={arXiv preprint arXiv:2602.16198},
  year={2026}
}

@article{wang2026coarse,
  title={Coarse-Guided Visual Generation via Weighted h-Transform Sampling},
  author={Wang, Yanghao and Jiang, Ziqi and Wang, Zhen and Chen, Long},
  journal={arXiv preprint arXiv:2603.12057},
  year={2026}
}

@article{jang2026self,
  title={Self-Refining Video Sampling},
  author={Jang, Sangwon and Ki, Taekyung and Jo, Jaehyeong and Xie, Saining and Yoon, Jaehong and Hwang, Sung Ju},
  journal={arXiv preprint arXiv:2601.18577},
  year={2026}
}

@article{chung2022diffusion,
  title={Diffusion posterior sampling for general noisy inverse problems},
  author={Chung, Hyungjin and Kim, Jeongsol and Mccann, Michael T and Klasky, Marc L and Ye, Jong Chul},
  journal={arXiv preprint arXiv:2209.14687},
  year={2022}
}

@article{wu2023practical,
  title={Practical and asymptotically exact conditional sampling in diffusion models},
  author={Wu, Luhuan and Trippe, Brian and Naesseth, Christian and Blei, David and Cunningham, John P},
  journal={Advances in Neural Information Processing Systems},
  volume={36},
  pages={31372--31403},
  year={2023}
}

@article{he2023manifold,
  title={Manifold preserving guided diffusion},
  author={He, Yutong and Murata, Naoki and Lai, Chieh-Hsin and Takida, Yuhta and Uesaka, Toshimitsu and Kim, Dongjun and Liao, Wei-Hsiang and Mitsufuji, Yuki and Kolter, J Zico and Salakhutdinov, Ruslan and others},
  journal={arXiv preprint arXiv:2311.16424},
  year={2023}
}

@article{singhal2025general,
  title={A general framework for inference-time scaling and steering of diffusion models},
  author={Singhal, Raghav and Horvitz, Zachary and Teehan, Ryan and Ren, Mengye and Yu, Zhou and McKeown, Kathleen and Ranganath, Rajesh},
  journal={arXiv preprint arXiv:2501.06848},
  year={2025}
}

@article{hou2024training,
  title={Training-free camera control for video generation},
  author={Hou, Chen and Chen, Zhibo},
  journal={arXiv preprint arXiv:2406.10126},
  year={2024}
}

@inproceedings{zhou2025latent,
  title={Latent-reframe: Enabling camera control for video diffusion models without training},
  author={Zhou, Zhenghong and An, Jie and Luo, Jiebo},
  booktitle={Proceedings of the IEEE/CVF International Conference on Computer Vision},
  pages={12779--12789},
  year={2025}
}

@article{park2025zero4d,
  title={Zero4D: Training-Free 4D Video Generation From Single Video Using Off-the-Shelf Video Diffusion},
  author={Park, Jangho and Kwon, Taesung and Ye, Jong Chul},
  journal={arXiv preprint arXiv:2503.22622},
  year={2025}
}

@article{yu2024viewcrafter,
  title={Viewcrafter: Taming video diffusion models for high-fidelity novel view synthesis},
  author={Yu, Wangbo and Xing, Jinbo and Yuan, Li and Hu, Wenbo and Li, Xiaoyu and Huang, Zhipeng and Gao, Xiangjun and Wong, Tien-Tsin and Shan, Ying and Tian, Yonghong},
  journal={arXiv preprint arXiv:2409.02048},
  year={2024}
}

@article{seo2025vid,
  title={Vid-camedit: Video camera trajectory editing with generative rendering from estimated geometry},
  author={Seo, Junyoung and Han, Jisang and Jung, Jaewoo and Jin, Siyoon and Lee, Joungbin and Narihira, Takuya and Fukuda, Kazumi and Shibuya, Takashi and Ahn, Donghoon and Hu, Shoukang and others},
  journal={arXiv preprint arXiv:2506.13697},
  year={2025}
}

@article{sreehari2016plug,
  title={Plug-and-play priors for bright field electron tomography and sparse interpolation},
  author={Sreehari, Suhas and Venkatakrishnan, S Venkat and Wohlberg, Brendt and Buzzard, Gregery T and Drummy, Lawrence F and Simmons, Jeffrey P and Bouman, Charles A},
  journal={IEEE Transactions on Computational Imaging},
  volume={2},
  number={4},
  pages={408--423},
  year={2016},
  publisher={IEEE}
}

@article{romano2017little,
  title={The little engine that could: Regularization by denoising (RED)},
  author={Romano, Yaniv and Elad, Michael and Milanfar, Peyman},
  journal={SIAM journal on imaging sciences},
  volume={10},
  number={4},
  pages={1804--1844},
  year={2017},
  publisher={SIAM}
}

@inproceedings{hu2023restoration,
  title={A restoration network as an implicit prior},
  author={Hu, Yuyang and Delbracio, Mauricio and Milanfar, Peyman and Kamilov, Ulugbek},
  booktitle={The Twelfth International Conference on Learning Representations},
  year={2023}
}

@inproceedings{terris2025fire,
  title={Fire: Fixed-points of restoration priors for solving inverse problems},
  author={Terris, Matthieu and Kamilov, Ulugbek S and Moreau, Thomas},
  booktitle={Proceedings of the Computer Vision and Pattern Recognition Conference},
  pages={23185--23194},
  year={2025}
}

@article{anderson1982reverse,
  title={Reverse-time diffusion equation models},
  author={Anderson, Brian DO},
  journal={Stochastic Processes and their Applications},
  volume={12},
  number={3},
  pages={313--326},
  year={1982},
  publisher={Elsevier}
}

@article{vincent2011connection,
  title={A connection between score matching and denoising autoencoders},
  author={Vincent, Pascal},
  journal={Neural computation},
  volume={23},
  number={7},
  pages={1661--1674},
  year={2011},
  publisher={MIT Press}
}

@article{liu2022flow,
  title={Flow straight and fast: Learning to generate and transfer data with rectified flow},
  author={Liu, Xingchao and Gong, Chengyue and Liu, Qiang},
  journal={arXiv preprint arXiv:2209.03003},
  year={2022}
}

@article{denker2024deft,
  title={DEFT: Efficient fine-tuning of diffusion models by learning the generalised $ h $-transform},
  author={Denker, Alexander and Vargas, Francisco and Padhy, Shreyas and Didi, Kieran and Mathis, Simon and Dutordoir, Vincent and Barbano, Riccardo and Mathieu, Emile and Komorowska, Urszula J and Lio, Pietro},
  journal={Advances in Neural Information Processing Systems},
  volume={37},
  pages={19636--19682},
  year={2024}
}

@article{bengio2013generalized,
  title={Generalized denoising auto-encoders as generative models},
  author={Bengio, Yoshua and Yao, Li and Alain, Guillaume and Vincent, Pascal},
  journal={Advances in neural information processing systems},
  volume={26},
  year={2013}
}

@inproceedings{lugmayr2022repaint,
  title={Repaint: Inpainting using denoising diffusion probabilistic models},
  author={Lugmayr, Andreas and Danelljan, Martin and Romero, Andres and Yu, Fisher and Timofte, Radu and Van Gool, Luc},
  booktitle={Proceedings of the IEEE/CVF conference on computer vision and pattern recognition},
  pages={11461--11471},
  year={2022}
}

@article{singer2025time,
  title={Time-to-Move: Training-Free Motion Controlled Video Generation via Dual-Clock Denoising},
  author={Singer, Assaf and Rotstein, Noam and Mann, Amir and Kimmel, Ron and Litany, Or},
  journal={arXiv preprint arXiv:2511.08633},
  year={2025}
}

@article{song2025worldforge,
  title={Taming Video Models for 3D and 4D Generation via Zero-Shot Camera Control},
  author={Song, Chenxi and Yang, Yanming and Zhao, Tong and Li, Ruibo and Zhang, Chi},
  journal={arXiv preprint arXiv:2509.15130},
  year={2025}
}

@inproceedings{rombach2022high,
  title={High-resolution image synthesis with latent diffusion models},
  author={Rombach, Robin and Blattmann, Andreas and Lorenz, Dominik and Esser, Patrick and Ommer, Bj{\"o}rn},
  booktitle={Proceedings of the IEEE/CVF conference on computer vision and pattern recognition},
  pages={10684--10695},
  year={2022}
}

@inproceedings{esser2021taming,
  title={Taming transformers for high-resolution image synthesis},
  author={Esser, Patrick and Rombach, Robin and Ommer, Bjorn},
  booktitle={Proceedings of the IEEE/CVF conference on computer vision and pattern recognition},
  pages={12873--12883},
  year={2021}
}

@article{liu1994covariance,
  title={Covariance structure of the Gibbs sampler with applications to the comparisons of estimators and augmentation schemes},
  author={Liu, Jun S and Wong, Wing Hung and Kong, Augustine},
  journal={Biometrika},
  pages={27--40},
  year={1994},
  publisher={JSTOR}
}

@article{roberts1997updating,
  title={Updating schemes, correlation structure, blocking and parameterization for the Gibbs sampler},
  author={Roberts, Gareth O and Sahu, Sujit K},
  journal={Journal of the Royal Statistical Society Series B: Statistical Methodology},
  volume={59},
  number={2},
  pages={291--317},
  year={1997},
  publisher={Oxford University Press}
}

@article{ye2024tfg,
  title={Tfg: Unified training-free guidance for diffusion models},
  author={Ye, Haotian and Lin, Haowei and Han, Jiaqi and Xu, Minkai and Liu, Sheng and Liang, Yitao and Ma, Jianzhu and Zou, James and Ermon, Stefano},
  journal={Advances in Neural Information Processing Systems},
  volume={37},
  pages={22370--22417},
  year={2024}
}

@article{polyak1992acceleration,
  title={Acceleration of stochastic approximation by averaging},
  author={Polyak, Boris T and Juditsky, Anatoli B},
  journal={SIAM journal on control and optimization},
  volume={30},
  number={4},
  pages={838--855},
  year={1992},
  publisher={SIAM}
}

@article{zhou2018stereo,
  title={Stereo magnification: Learning view synthesis using multiplane images},
  author={Zhou, Tinghui and Tucker, Richard and Flynn, John and Fyffe, Graham and Snavely, Noah},
  journal={arXiv preprint arXiv:1805.09817},
  year={2018}
}

@article{unterthiner2018towards,
  title={Towards accurate generative models of video: A new metric \& challenges},
  author={Unterthiner, Thomas and Van Steenkiste, Sjoerd and Kurach, Karol and Marinier, Raphael and Michalski, Marcin and Gelly, Sylvain},
  journal={arXiv preprint arXiv:1812.01717},
  year={2018}
}

@inproceedings{li2025megasam,
  title={Megasam: Accurate, fast and robust structure and motion from casual dynamic videos},
  author={Li, Zhengqi and Tucker, Richard and Cole, Forrester and Wang, Qianqian and Jin, Linyi and Ye, Vickie and Kanazawa, Angjoo and Holynski, Aleksander and Snavely, Noah},
  booktitle={Proceedings of the IEEE/CVF Conference on Computer Vision and Pattern Recognition},
  pages={10486--10496},
  year={2025}
}

@article{wan2025wan,
  title={Wan: Open and advanced large-scale video generative models},
  author={Wan, Team and Wang, Ang and Ai, Baole and Wen, Bin and Mao, Chaojie and Xie, Chen-Wei and Chen, Di and Yu, Feiwu and Zhao, Haiming and Yang, Jianxiao and others},
  journal={arXiv preprint arXiv:2503.20314},
  year={2025}
}

@article{welford1962note,
  title={Note on a method for calculating corrected sums of squares and products},
  author={Welford, Barry Payne},
  journal={Technometrics},
  volume={4},
  number={3},
  pages={419--420},
  year={1962},
  publisher={Taylor \& Francis}
}

@article{luo2016understanding,
  title={Understanding the effective receptive field in deep convolutional neural networks},
  author={Luo, Wenjie and Li, Yujia and Urtasun, Raquel and Zemel, Richard},
  journal={Advances in neural information processing systems},
  volume={29},
  year={2016}
}

@article{kynkaanniemi2024applying,
  title={Applying guidance in a limited interval improves sample and distribution quality in diffusion models},
  author={Kynk{\"a}{\"a}nniemi, Tuomas and Aittala, Miika and Karras, Tero and Laine, Samuli and Aila, Timo and Lehtinen, Jaakko},
  journal={Advances in Neural Information Processing Systems},
  volume={37},
  pages={122458--122483},
  year={2024}
}

@inproceedings{radford2021learning,
  title={Learning transferable visual models from natural language supervision},
  author={Radford, Alec and Kim, Jong Wook and Hallacy, Chris and Ramesh, Aditya and Goh, Gabriel and Agarwal, Sandhini and Sastry, Girish and Askell, Amanda and Mishkin, Pamela and Clark, Jack and others},
  booktitle={International conference on machine learning},
  pages={8748--8763},
  year={2021},
  organization={PMLR}
}

@inproceedings{lu2023contrastive,
  title={Contrastive energy prediction for exact energy-guided diffusion sampling in offline reinforcement learning},
  author={Lu, Cheng and Chen, Huayu and Chen, Jianfei and Su, Hang and Li, Chongxuan and Zhu, Jun},
  booktitle={International Conference on Machine Learning},
  pages={22825--22855},
  year={2023},
  organization={PMLR}
}

@inproceedings{song2021scorebased,
  title={Score-Based Generative Modeling through Stochastic Differential Equations},
  author={Yang Song and Jascha Sohl-Dickstein and Diederik P Kingma and Abhishek Kumar and Stefano Ermon and Ben Poole},
  booktitle={The Ninth International Conference on Learning Representations},
  year={2021}
}

@inproceedings{zhang2018lpips,
  title={The Unreasonable Effectiveness of Deep Features as a Perceptual Metric},
  author={Zhang, Richard and Isola, Phillip and Efros, Alexei A. and Shechtman, Eli and Wang, Oliver},
  booktitle={Proceedings of the IEEE conference on computer vision and pattern recognition},
  pages={586--595},
  year={2018}
}

@inproceedings{ge2024content,
  title={On the Content Bias in {Frechet Video Distance}},
  author={Ge, Songwei and Mahapatra, Aniruddha and Parmar, Gaurav and Zhu, Jun-Yan and Huang, Jia-Bin},
  booktitle={Proceedings of the IEEE/CVF Conference on Computer Vision and Pattern Recognition},
  pages={7277--7288},
  year={2024}
}

@article{liu2024fvmd,
  title={{Fr\'echet Video Motion Distance}: A Metric for Evaluating Motion Consistency in Videos},
  author={Liu, Jiahe and Qu, Youran and Yan, Qi and Zeng, Xiaohui and Wang, Lele and Liao, Renjie},
  journal={arXiv preprint arXiv:2407.16124},
  year={2024}
}

@article{xu2026visionreward,
  title={{VisionReward}: Fine-Grained Multi-Dimensional Human Preference Learning for Image and Video Generation},
  author={Xu, Jiazheng and Huang, Yu and Cheng, Jiale and Yang, Yuanming and Xu, Jiajun and Wang, Yuan and Duan, Wenbo and Yang, Shen and Jin, Qunlin and Li, Shurun and Teng, Jiayan and Yang, Zhuoyi and Zheng, Wendi and Liu, Xiao and Zhang, Dan and Ding, Ming and Zhang, Xiaohan and Huang, Shiyu and Gu, Xiaotao and Huang, Minlie and Tang, Jie and Dong, Yuxiao},
  journal={Proceedings of the AAAI Conference on Artificial Intelligence},
  volume={40},
  number={13},
  pages={11269--11277},
  year={2026},
  doi={10.1609/aaai.v40i13.38107}
}

@article{wang2004ssim,
  title={Image quality assessment: from error visibility to structural similarity},
  author={Wang, Zhou and Bovik, Alan C. and Sheikh, Hamid R. and Simoncelli, Eero P.},
  journal={IEEE transactions on image processing},
  volume={13},
  number={4},
  pages={600--612},
  year={2004},
  publisher={IEEE}
}


\newpage
\appendix
\section{Extended Related Work and Positioning}
\label{sec:appendix-related}

This appendix expands Section~\ref{sec:related} with a technical positioning
of \textit{GibbsLens} against the four research lines it sits between: camera-controllable
video generation, training-free conditional sampling on pretrained
diffusion/flow models, plug-and-play denoiser-as-prior methods, and generalized
denoising auto-encoders.

\subsection{Camera-controllable video generation}
\label{sec:appendix-related-camera}

\paragraph{Training-based controllers.}
CameraCtrl~\citep{he2024cameractrl} and
CameraCtrl~II~\citep{he2025cameractrl} inject Pl\"ucker-coordinate camera
embeddings into a CogVideoX-style backbone.
VD3D~\citep{bahmani2024vd3d} and AC3D~\citep{bahmani2025ac3d} train video DiTs
with explicit 3D-aware tokenization. MotionCtrl~\citep{wang2024motionctrl}
trains a separate motion module conditioned on camera and object motion
fields. Multi-view extensions: SyncCammaster~\citep{bai2024syncammaster}
synchronizes generation across views, and
ReCamMaster~\citep{bai2025recammaster} learns to re-render an input video
under a new trajectory by joint training on synchronized multi-camera data.
Several systems build a 4D scaffold from a source video and guide the
generator with it: TrajectoryAttention~\citep{xiao2024trajectory} introduces
trajectory-aware attention; TrajectoryCrafter~\citep{yu2025trajectorycrafter}
renders explicit point clouds from monocular input and trains a refinement
diffusion model on top. ViewCrafter~\citep{yu2024viewcrafter} and
Vid-camedit~\citep{seo2025vid} solve sparse-view novel-view synthesis by
combining explicit 3D priors with generative refinement;
\citet{van2024generative} provides a generative-modeling perspective on the
same problem. The common requirement across this family is that at least one
component --- backbone, adapter, or refinement head --- is fine-tuned to
internalize the trajectory-to-video correspondence.

\paragraph{Training-free controllers.}
TTM~\citep{singer2025time} and
WorldForge~\citep{song2025worldforge} construct a warped guidance video by
lifting the source to a point cloud through monocular depth and reprojecting
under the target camera, then run a pretrained generator with hard latent
replacement on the visibility mask at every denoising step.
Recapture~\citep{zhang2025recapture} and Latent
Reframing~\citep{zhou2025latent} use a similar warp-and-replace primitive with
mask construction tuned for static scenes. Zero4D~\citep{park2025zero4d}
extends this to free-camera 4D generation in a zero-shot fashion, again via
masked replacement. Coarse-Guided~\citep{wang2026coarse} drops the partial
mask and instead applies a global noise-aware soft pull toward the warped
reference. Both ends of this spectrum are degenerate boundary limits of
\textit{GibbsLens}: at $\sigma_y\!\to\!0$ the warp likelihood concentrates on
$\{z_0:M\!\odot\!z_0=M\!\odot\!\tilde z_0\}$ and the write-back collapses
to hard replacement; at $M\!\equiv\!1$ the partial mask becomes the whole
grid and the soft pull reduces to Coarse-Guided. These are sanity-check
consistency properties, not the regime where \textit{GibbsLens} operates. The novel object is the interior: at finite
$\sigma_y$ on a non-trivial partial mask, the inner block-conditional
Gibbs refinement targets the partial-observation conditional posterior
$p(z_0|_{1-M}\!\mid\!\bar z_t|_M)$
(Proposition~\ref{prop:gibbs-stationary}), which has no counterpart in
either baseline and is what delivers the gains in
Section~\ref{sec:exp-sota}.

\subsection{Training-free conditional sampling}
\label{sec:appendix-related-condsampling}

\paragraph{First-order surrogates: DPS family.}
Diffusion Posterior Sampling~\citep{chung2022diffusion} replaces the
intractable $h$-induced drift with the gradient of the observation likelihood
evaluated at the Tweedie mean $\hat z_0$. A line of follow-ups
substitutes alternative point estimates for $\hat z_0$ or alternative
likelihoods (FreeDoM, $\Pi$GDM, MPGD~\citep{he2023manifold}, TFG-UGD), all
sharing the property that the surrogate is a single point evaluation of
$\nabla\log h$. \emph{Position of \textit{GibbsLens}:} the outer step is exactly a
DPS-style surrogate \emph{restricted to the partial mask $M$}; the structural
addition is the inner block-conditional Gibbs refinement on $1\!-\!M$.

\paragraph{Noise-aware weighted surrogates.}
The DPS surrogate is most reliable at high $\sigma_t$ and degrades as
$\sigma_t\!\to\!0$, where $\hat z_0$ becomes deterministic. Weighted
$h$-transform sampling~\citep{wang2026coarse} introduces a noise-level-aware
scalar $\lambda_{\sigma_t}$ that down-weights the surrogate at low noise.
DEFT~\citep{denker2024deft} amortizes the surrogate by training a small
network that predicts the $h$-induced drift directly, and \citet{zhu2026training}
fine-tunes the base model to internalize the same drift. \emph{Position of
\textit{GibbsLens}:} we extend the global scalar $\lambda_{\sigma_t}$ to a
\emph{spatially non-uniform} mask $M$ and pair it with a novel inner refinement on
the unobserved support, while keeping the conditioning entirely at inference
time --- no extra network and no fine-tuning.

\paragraph{Sequential Monte Carlo and Feynman--Kac.}
Twisted Diffusion Samplers~\citep{wu2023practical} and Feynman--Kac
steering~\citep{singhal2025general} maintain a population of trajectories and
reweight them by the running terminal compatibility, recovering an unbiased
estimator of $\nabla\log h_t$ in the large-particle limit. The price is
$N\!\times$ inference compute per step. \emph{Position of \textit{GibbsLens}:} the
inner Gibbs chain is a single-trajectory variance-reduction primitive
operating on the same physical sample, achieving similar conditional
adherence at a small per-step multiplier rather than $N\!\times$.

\subsection{Plug-and-play and denoiser-as-prior methods}
\label{sec:appendix-related-pnp}

PnP-ADMM~\citep{sreehari2016plug}, RED~\citep{romano2017little}, and
score-based extensions~\citep{hu2023restoration,terris2025fire} share the
``denoiser as implicit prior'' perspective: a pretrained denoiser is
invoked at a single fixed noise level inside a half-quadratic or proximal
solver, acting as a projection-onto-prior step against a hard
data-fidelity constraint. \textit{GibbsLens} reuses the same intuition but in a
structurally different setting --- a multi-step flow-matching sampler with
a time-varying schedule, where the implicit prior is queried at every
$\sigma_t$ and conditioned on a partial observation, so the inner chain
targets $p(z_0|_{1-M}\!\mid\!\bar z_t|_M)$ rather than the unconditional
data law that PnP targets.

\section{Flow matching and Doob's $h$-transform background}
\label{sec:appendix-doob-background}

This appendix expands the brief summary in Section~\ref{sec:prelim} into a
self-contained derivation, starting from score-based diffusion as a forward
SDE, deriving the velocity--score identity that underwrites flow matching,
and arriving at the controlled flow-matching velocity used in
Eq.~\eqref{eq:m-vctrl} of the method section. Readers fluent in
score-based diffusion and Doob's $h$-transform can skip directly to
Section~\ref{sec:appendix-surrogates}.

\subsection{Score-based diffusion via SDEs}
\label{sec:appendix-sde}

Score-based generative models~\citep{song2021scorebased} cast sampling as
the time reversal of a noising SDE that transports data $z_0\!\sim\!p_0$
to a tractable prior $p_T\!=\!\mathcal N(0,I)$:
\begin{equation}
\label{eq:forward-sde}
    dz_t \;=\; f(z_t,t)\,dt \;+\; g(t)\,dw_t, \qquad t\in[0,T],
\end{equation}
with drift $f$, diffusion coefficient $g$, and Wiener process $w_t$. Its
time-reversed dynamics~\citep{anderson1982reverse}, integrated from
$z_T\!\sim\!p_T$ back to $t\!=\!0$, recover samples from $p_0$:
\begin{equation}
\label{eq:reverse-sde}
    dz_t \;=\; \bigl[f(z_t,t)-g(t)^2\,\nabla_{z_t}\log p_t(z_t)\bigr]\,dt
    \;+\; g(t)\,d\bar w_t,
\end{equation}
where $p_t$ is the marginal of $z_t$ and $\bar w_t$ a reverse-time Wiener
process. The score $\nabla_{z_t}\log p_t$ is learned as $s_\theta(z_t,t)$
via denoising score matching~\citep{vincent2011connection}. The same
marginals $\{p_t\}$ can be induced by the deterministic probability-flow
ODE~\citep{song2021scorebased},
\begin{equation}
\label{eq:pf-ode}
    dz_t \;=\; \bigl[f(z_t,t)-\tfrac12\,g(t)^2\,\nabla_{z_t}\log p_t(z_t)\bigr]\,dt.
\end{equation}

\subsection{Flow matching as reparameterization, and the velocity--score identity}
\label{sec:appendix-vel-score}

Flow matching~\citep{lipman2022flow,liu2022flow} is equivalent to a
reparameterization of the probability-flow ODE Eq.~\eqref{eq:pf-ode}.
Choosing the variance-exploding coefficients~\citep{song2021scorebased}
$f(z_t,t)\!=\!(\dot\sigma_t/\sigma_t)\,z_t$ and
$g(t)\!=\!\sqrt{2\,\dot\sigma_t\,\sigma_t}$ in
Eq.~\eqref{eq:forward-sde} produces marginals
$z_t\!=\!(1\!-\!\sigma_t)z_0+\sigma_t\epsilon$ with
$\epsilon\!\sim\!\mathcal N(0,I)$ along a monotone schedule
$\sigma_0\!=\!0,\sigma_1\!=\!1$. Substituting these into
Eq.~\eqref{eq:pf-ode} reduces it to the deterministic ODE
$dz_t/dt\!=\!u_\theta(z_t,t,c)$ through the \emph{velocity--score identity}
\begin{equation}
\label{eq:vel-score-identity}
    u_\theta(z_t,t,c)
    \;=\;
    \frac{\dot\sigma_t}{\sigma_t}\,z_t
    \;-\;
    \dot\sigma_t\,\sigma_t\,\nabla_{z_t}\log p_t(z_t\mid c),
\end{equation}
where $c$ is the external condition (e.g., text for T2V). Flow matching
trains $u_\theta$ to regress $\epsilon\!-\!z_0$; sampling integrates the
ODE backward from $z_1\!\sim\!\mathcal N(0,I)$ to $z_0$.
Eq.~\eqref{eq:vel-score-identity} is the load-bearing bridge in
\textit{GibbsLens}'s derivation: every step that adds an $h$-induced drift to the
score in Eq.~\eqref{eq:h-conditioned-reverse-sde} can be rewritten as an
additive correction to $u_\theta$, giving the controlled velocity
Eq.~\eqref{eq:h-controlled-velocity}. The plug-in clean-prediction
estimator Eq.~\eqref{eq:fm-cleanpred} follows from
Eq.~\eqref{eq:vel-score-identity} by Tweedie's identity
$\hat z_0(z_t)\!=\!z_t+\sigma_t^2\nabla_{z_t}\log p_t(z_t)$ and the
forward parameterization above.

\subsection{Doob's $h$-transform for conditional sampling}
\label{sec:appendix-doob-h}

Many practical tasks impose an \emph{endpoint constraint} on the reverse
SDE rather than just a terminal distribution --- for instance, the noisy
inverse problem $y\!=\!\mathcal A(z_0)+n$ with known operator
$\mathcal A$, observation $y$, and noise $n$. Conditioning the marginals
$p_t$ in Eq.~\eqref{eq:reverse-sde} on such an event is intractable, since
the conditional dynamics couples through all future steps. Doob's
$h$-transform~\citep{rogers2000diffusions} resolves this by tilting the
unconditioned path measure with a non-negative terminal weight $h(z_0)$
--- e.g.\ the indicator $\mathbf 1_B(z_0)$ for hard conditioning, or the
likelihood $p(y\!\mid\!z_0)$ for soft conditioning on $y$ --- yielding a
Markov diffusion whose endpoint follows the desired conditional and whose
drift carries an additive correction.

Following~\citet{denker2024deft,wang2026coarse}, define the time-dependent
$h$-function as the conditional expectation of the terminal weight:
\begin{equation}
\label{eq:h-function}
    h_t(z_t)
    \;=\;
    \mathbb E\!\bigl[\,h(z_0)\,\bigm|\,z_t\bigr]
    \;=\;
    \int p_{0|t}(z_0\!\mid\!z_t)\,h(z_0)\,dz_0,
\end{equation}
where $p_{0|t}$ is the unconditional posterior of the clean endpoint given
the current state. Doob's theorem yields the conditioned reverse SDE
Eq.~\eqref{eq:h-conditioned-reverse-sde} with an additive
\emph{$h$-induced drift} $\nabla_{z_t}\log h_t(z_t)$, and via
Eq.~\eqref{eq:vel-score-identity} the controlled flow-matching velocity
Eq.~\eqref{eq:h-controlled-velocity}. For the noisy inverse problem with
$h(z_0)\!=\!p(y\!\mid\!z_0)$, Eq.~\eqref{eq:h-function} specializes to
$h_t(z_t)\!=\!p_t(y\!\mid\!z_t)$, recovering the standard
posterior-sampling identity
$\nabla_{z_t}\log p_t(z_t\!\mid\!y)\!=\!\nabla_{z_t}\log p_t(z_t)+\nabla_{z_t}\log p_t(y\!\mid\!z_t)$~\citep{denker2024deft}.

\subsection{Tractable surrogates for $\nabla\log h_t$}
\label{sec:appendix-surrogates}

The $h$-induced drift $\nabla_{z_t}\log h_t(z_t)$ in
Eq.~\eqref{eq:h-conditioned-reverse-sde} is intractable in general, since
evaluating the $h$-function Eq.~\eqref{eq:h-function} requires the
unconditional posterior $p_{0|t}$ of a large pretrained model and a
marginalization over future trajectories. Practical conditional sampling
therefore replaces $\nabla\log h_t$ with a tractable surrogate. The
useful taxonomic axis turns out to be \emph{whether the surrogate
backpropagates through the denoiser to compute its gradient}; we group
the established families along this axis and locate \textit{GibbsLens} among
them.

\paragraph{Common first-order point surrogate.}
All single-trajectory families below share the same starting point:
replace the intractable expectation in Eq.~\eqref{eq:h-function} by $h$
evaluated at the Tweedie posterior mean $\hat z_0(z_t)$
(Eq.~\eqref{eq:fm-cleanpred}),
\begin{equation}
\label{eq:dps-surrogate}
    h_t(z_t)\;\approx\;h\bigl(\hat z_0(z_t)\bigr)
    \;\Longrightarrow\;
    \nabla_{z_t}\log h_t(z_t)
    \;\approx\;
    \nabla_{z_t}\log h\bigl(\hat z_0(z_t)\bigr).
\end{equation}
The chain rule expands the right-hand side as
\begin{equation}
\label{eq:dps-chain}
    \nabla_{z_t}\log h\bigl(\hat z_0(z_t)\bigr)
    \;=\;
    \underbrace{\bigl(\nabla_{\hat z_0}\log h\bigr)\!\bigl(\hat z_0(z_t)\bigr)}_{\text{cheap, no autograd}}
    \,\cdot\,
    \underbrace{\nabla_{z_t}\hat z_0(z_t)}_{=\,I-\sigma_t\,\nabla_{z_t}u_\theta(z_t,t,c)},
\end{equation}
so a faithful evaluation requires the denoiser Jacobian
$\nabla_{z_t}u_\theta$. Whether this Jacobian is computed in full or
dropped is what distinguishes families (i) and (ii) below.

\paragraph{(i) Jacobian-aware first-order surrogates (gradient through the denoiser).}
This family computes Eq.~\eqref{eq:dps-chain} in full, backpropagating
through one or more denoiser forwards. Diffusion Posterior Sampling
(DPS)~\citep{chung2022diffusion} is the canonical instance: for a
Gaussian observation $h(z_0)\!=\!\mathcal N(y;\mathcal A(z_0),\tau^2 I)$
the surrogate becomes the gradient of a measurement-consistency loss
with respect to $z_t$, computed by autograd through $\hat z_0(z_t)$ and
hence through $u_\theta$. Manifold Preserving Guided Diffusion
(MPGD)~\citep{he2023manifold} additionally projects the
Jacobian-computed gradient onto the data-manifold tangent to prevent
off-manifold drift. Training-Free Guidance with universal-guidance
recurrence (TFG-UGD)~\citep{ye2024tfg} runs several inner mean-guidance
iterations per outer step, each backpropagating through the denoiser.
The defining feature of this family is the autograd-through-$u_\theta$
requirement: each surrogate evaluation costs at least one
denoiser-sized backward pass and the activations must be retained for
backprop, so the memory footprint scales with the model and resolution.
This is the price paid for using the full first-order Taylor expansion
of $\log h$ around $z_t$.

\paragraph{(ii) Jacobian-free first-order surrogates (stop-grad through the denoiser).}
This family uses the same point surrogate Eq.~\eqref{eq:dps-surrogate}
but drops the Jacobian in Eq.~\eqref{eq:dps-chain}: $\nabla_{z_t}\hat z_0
\!\approx\!I$, equivalent to detaching $\hat z_0$ from $z_t$ before
gradient computation. The surrogate then collapses to
$\nabla_{z_t}\log h(\hat z_0)\!\approx\!\nabla_{\hat z_0}\log h(\hat z_0)$,
which is cheap and requires no backprop through $u_\theta$. This
``stop-grad'' approximation is documented in DPS itself
(see~\citealt{chung2022diffusion}, App.~A) and is the standard choice
for inference-time guidance on large pretrained backbones where the
Jacobian-aware variant is computationally infeasible. Two strategies
within this family compensate for the loss of Jacobian information.
\emph{(a) Noise-aware weighting.} Weighted $h$-transform
sampling~\citep{wang2026coarse} introduces
$\lambda_{\sigma_t}\!=\!\sigma_t^{\alpha}$ on the surrogate,
\begin{equation}
\label{eq:weighted-h}
    \nabla_{z_t}\log h_t(z_t)
    \;\approx\;
    \lambda_{\sigma_t}\,\nabla_{\hat z_0}\log h\bigl(\hat z_0(z_t)\bigr),
\end{equation}
down-weighting at low $\sigma_t$ where the Jacobian carries the most
information and the stop-grad approximation degrades fastest. For
coarse-guided flow-matching with
$h(z_0)\!\propto\!\delta(z_0\!-\!\tilde y)$ replaced by a tractable
proxy $\tilde y$, this yields the closed-form controlled velocity
\begin{equation}
\label{eq:weighted-flow-velocity}
    u^{\text{ctrl}}_t
    \;=\;
    u_\theta
    \;+\;
    \lambda_{\sigma_t}\!\Bigl(\frac{z_t-\tilde y}{\sigma_t}-u_\theta\Bigr),
\end{equation}
interpolating between the pretrained velocity
($\lambda_{\sigma_t}\!\to\!0$) and a hard pull toward the warped
reference ($\lambda_{\sigma_t}\!\to\!1$). \emph{(b) Amortized
surrogate.} DEFT~\citep{denker2024deft} trains a small network to
predict $\nabla_{z_t}\log h_t$ directly, sidestepping the inference-time
Jacobian by amortizing it offline; \citet{zhu2026training} fine-tunes
the base model itself to internalize the same correction. Both incur an
extra training stage but, like the weighted variant, avoid autograd
through the pretrained denoiser at sampling time.

\paragraph{(iii) Particle-based surrogates (Sequential Monte Carlo / Feynman--Kac).}
Twisted Diffusion Samplers~\citep{wu2023practical} and Feynman--Kac
steering~\citep{singhal2025general} maintain a population of
trajectories and reweight them by the running terminal compatibility,
recovering an unbiased estimate of $\nabla\log h_t$ in the
large-particle limit at the cost of $N\!\times$ inference compute per
step. They are orthogonal to (i)/(ii) along the Jacobian axis: the bias
of the surrogate is reduced by particle resampling rather than by a
better single-point estimator, so neither the Jacobian nor a stop-grad
proxy enters the picture.

\subsection{Position of our work}
\label{sec:appendix-position}
\textit{GibbsLens} belongs to family (ii) (Jacobian-free first-order
surrogate): the surrogate is the Tweedie point estimate
Eq.~\eqref{eq:dps-surrogate} with the denoiser Jacobian
$\nabla_{z_t}u_\theta$ dropped. The Jacobian-free choice is mandatory
rather than optional in our setting --- the Wan~2.2 backbone has
$\sim\!5$B parameters at $720\!\times\!1280\!\times\!49$ resolution, so
each backward through $u_\theta$ would dominate the per-step compute of
any Jacobian-aware surrogate from family~(i). What \textit{GibbsLens} adds on
top of the existing family~(ii) is a \emph{spatially non-uniform}
confidence map: we replace the global scalar weight $\lambda_{\sigma_t}$
of Eq.~\eqref{eq:weighted-h} by the warp-derived mask $M$ over latent
sites, so the surrogate is applied at full strength on observed sites
and zero on $1\!-\!M$, recovering most of the Jacobian's
spatial-information content at zero autograd cost. The complementary
inner block-conditional Gibbs refinement on $1\!-\!M$
(Section~\ref{sec:method}) is what closes the gap on the unobserved
support that family~(ii) leaves open by construction.

\section{Block-conditional pseudo-Gibbs sampling}
\label{sec:appendix-complement-theory}

This appendix reads \textit{GibbsLens}'s inner refinement loop as a
\emph{block-conditional pseudo-Gibbs Markov chain}: we recall the
generalized DAE / pseudo-Gibbs framework of
\citet{bengio2013generalized}, prove Proposition~\ref{prop:gibbs-stationary}
as its sub-state extension targeting the partial-observation conditional
law, and contrast the resulting chain with SRVS~\citep{jang2026self}.
All notation follows
Section~\ref{sec:prelim}: the pretrained latent flow-matching model induces
a velocity field $u_\theta(z_k,t_k,c)$ and a clean-prediction estimator
$\hat z_0(z_k,t_k,c) = z_k - \sigma_k\,u_\theta(z_k,t_k,c)$ via
Eq.~\eqref{eq:fm-cleanpred}.

\subsection{Generalized denoising auto-encoders as generative models}
\label{sec:appendix-bengio}

Let $X$ denote the clean random variable and $\tilde X$ its corruption under
a fixed conditional $\mathcal C(\tilde X\mid X)$.
\citet{bengio2013generalized} define the \emph{pseudo-Gibbs} Markov chain
associated with a learned reconstruction conditional
$P_\theta(X\mid \tilde X)$ as
\begin{equation}
\label{eq:bengio-chain}
X^{(j)}\;\sim\;P_\theta(X\mid\tilde X^{(j-1)}),
\qquad
\tilde X^{(j)}\;\sim\;\mathcal C(\tilde X\mid X^{(j)}),
\end{equation}
with transition operator
$T(X^{(j)}\mid X^{(j-1)})=\int P_\theta(X^{(j)}\mid \tilde X)\,\mathcal C(\tilde X\mid X^{(j-1)})\,d\tilde X$
on $X$. Their Theorem~1 states that if (i)~$P_\theta$ is a consistent
estimator of the true posterior $p(X\mid\tilde X)$ induced by the joint
$p(X)\,\mathcal C(\tilde X\mid X)$, and (ii)~$T$ is ergodic
(positivity of $P_\theta$ and $\mathcal C$ on a bounded volume is
sufficient), then the asymptotic distribution of Eq.~\eqref{eq:bengio-chain}
converges to the data-generating distribution $p(X)$.

Translated to a latent flow-matching generator at fixed sampling step $k$,
the identifications $\mathcal C\!\leftrightarrow\!R_{\sigma_k}$
(re-noising at level $\sigma_k$) and
$P_\theta\!\leftrightarrow\!D_{\sigma_k}$
(the implicit conditional induced by the denoiser/clean-prediction
estimator $\hat z_0(\cdot,t_k,c)$) give a perturb-and-redenoise iteration
on $z_k$ at fixed $\sigma_k$. Section~\ref{sec:appendix-complement-as-block}
extends this primitive to the partial-observation \emph{conditional}
setting that \textit{GibbsLens}'s inner loop targets.

\subsection{Sub-state ergodicity and stationary distribution}
\label{sec:appendix-complement-as-block}

This subsection proves Proposition~\ref{prop:gibbs-stationary} from
Section~\ref{sec:gibbs-refinement} by reduction to a sub-state instance of
\citet{bengio2013generalized} Theorem~1. Throughout we hold $\sigma_t > 0$,
mask $M$, warped latent $\tilde z_0$, and pin
$\bar z_t = (1-\sigma_t)\tilde z_0 + \sigma_t\,\xi_{\rm obs}$ fixed; we
write $z|_M$ and $z|_{1-M}$ for the corresponding restrictions of any
latent tensor.

\begin{lemma}[Sub-state DAE chain]
\label{lem:substate-dae}
With $\bar z_t|_M$ held as fixed exogenous context and $\sigma_t$ held
fixed, the inner-loop chain
Eqs.~\eqref{eq:m-perturb}--\eqref{eq:m-redenoise} restricted to coordinates
in $1\!-\!M$ is a generalized DAE chain in the sense of
\citet{bengio2013generalized}, with substitutions
$X\!\leftrightarrow\!z_0|_{1-M}$ and
$\tilde X\!\leftrightarrow\!\tilde z|_{1-M}$.
\end{lemma}

\begin{proof}[Proof of Lemma~\ref{lem:substate-dae}]
The corruption restricted to $1\!-\!M$ is
$\mathcal C(\tilde z|_{1-M}\mid z_0|_{1-M})=\mathcal N\!\bigl((1-\sigma_t)\,z_0|_{1-M},\,\sigma_t^2\,I\bigr)$,
a full-support Gaussian on $z_0|_{1-M}$. The conditional reconstruction
$P_\theta(z_0|_{1-M}\mid\tilde z|_{1-M},\,\tilde z|_M{=}\bar z_t|_M)$ is the
implicit conditional induced by the denoiser $u_\theta(\cdot,t,c)$ applied
to the perturbed iterate with the observed coordinates pinned at
$\bar z_t|_M$. Coordinates in $M$ are never re-noised during the inner loop;
they remain pinned to $\bar z_t|_M$, so the chain on $1\!-\!M$ is a
generalized DAE chain with $\bar z_t|_M$ as fixed exogenous context.
\end{proof}

\begin{proof}[Proof of Proposition~\ref{prop:gibbs-stationary}]
By Lemma~\ref{lem:substate-dae}, the chain on $1\!-\!M$ is a generalized
DAE chain with $\bar z_t|_M$ as fixed exogenous context.
\citet{bengio2013generalized} Theorem~1 states that if the reconstruction
conditional $P_\theta$ is consistent with the true posterior
$p(X\mid\tilde X)$ induced by the joint $p(X)\,\mathcal C(\tilde X\mid X)$,
and the transition operator is ergodic on a bounded volume, then the
chain's asymptotic distribution converges to the data-generating
distribution $p(X)$. In our setting, the joint at level $\sigma_t$ is
conditioned on $\bar z_t|_M$, so
$p(X)\!\leftrightarrow\!p(z_0|_{1-M}\mid \bar z_t|_M)$. Under assumption
(A1), $P_\theta\!\leftrightarrow\!u_\theta$ is consistent with this
conditional joint; under (A2), positivity of the perturbation kernel on
$1\!-\!M$ gives ergodicity. Therefore the clean-prediction iterates
converge in distribution to $p(z_0|_{1-M}\mid \bar z_t|_M)$.
\end{proof}

\paragraph{Connection to coordinate descent / block Gibbs sampling.}
The chain restricts the perturb-and-redenoise primitive to a fixed
coordinate block $1\!-\!M$ while holding $M$ frozen. This is the
generative-modeling analog of block-coordinate descent in optimization, in
which a fixed subset of variables is updated at each iteration with the rest
held constant; here ``descent'' on a loss is replaced by ``mixing'' of an
implicit prior. The use of a fixed block (rather than random or systematic
scans across blocks) is a deliberate task-driven choice: the partition
$(1\!-\!M, M)$ is given by the camera-control mask geometry and is the same
at every probe $j$.

\subsection{Comparison with SRVS}
\label{sec:appendix-srvs-diffs}

The pseudo-Gibbs chain primitive is shared with Self-Refining Video
Sampling (SRVS,~\citet{jang2026self}). At a fixed denoising step with
noise level $\sigma_k$, SRVS iterates
\begin{equation}
\label{eq:srvs-chain}
z_k^{(j)} \;=\; (1-\sigma_k)\,\hat z_0^{(j-1)} \;+\; \sigma_k\,\xi^{(j)},\qquad
\hat z_0^{(j)} \;=\; z_k^{(j)} \;-\; \sigma_k\,u_\theta(z_k^{(j)},t_k,c),
\qquad \xi^{(j)}\!\sim\!\mathcal N(0,I).
\end{equation}
With identifications $X\!\leftrightarrow\!\hat z_0$,
$\tilde X\!\leftrightarrow\!z_k$, $\mathcal C\!\leftrightarrow\!R_{\sigma_k}$
and $P_\theta\!\leftrightarrow\!D_{\sigma_k}$, Eq.~\eqref{eq:srvs-chain}
is a structural instance of Eq.~\eqref{eq:bengio-chain}: SRVS uses
Bengio's pseudo-Gibbs primitive at the chain level. The original SRVS
paper invokes this informally --- it does not state or apply Theorem~1
of \citet{bengio2013generalized} explicitly --- and motivates the
practice as ``mode-seeking'' toward the data manifold. Beyond the chain
primitive, SRVS adds an outer selective-refinement scheme applied
between outer steps as a per-iteration mixing rule on the next-level
noisy latent $z_{t_{i+1}}^{(k)}$, with the eligibility mask derived
from the lag-1 self-disagreement statistic
$U(z^{(k-1)},z^{(k)})\!=\!\frac{1}{C}\|D_\theta(z^{(k-1)})-D_\theta(z^{(k)})\|_1$
thresholded at a fixed $\tau$. Both the inner chain and the outer
mixing rule act on every coordinate of the latent; SRVS does not
condition the chain on any external observation.

\paragraph{Where \textit{GibbsLens} differs.}
\textit{GibbsLens} extends Bengio's primitive in five non-degenerate respects.

\emph{(i) Conditional vs.\ unconditional target (theorem correctness).}
The inner chain runs on the unobserved sub-state $z|_{1-M}$ with the
observed coordinates pinned at the freshly noised observation
$\bar z_t|_M=(1-\sigma_t)\tilde z_0|_M+\sigma_t\,\xi_{\rm obs}|_M$.
Proposition~\ref{prop:gibbs-stationary} extends Bengio's Theorem~1 to
the sub-state setting and proves that the chain's stationary
distribution is the partial-observation conditional data law
$p(z_0|_{1-M}\!\mid\!\bar z_t|_M)$. SRVS targets the unconditional
$p(z_0)$ informally without a formal theorem statement or proof;
\textit{GibbsLens}'s claim is sharper, conditioned on the observed evidence,
and proven.

\emph{(ii) Block-conditional vs.\ full-state chain (the algorithmic novelty).}
\textit{GibbsLens} acts on the fixed coordinate block $1\!-\!M$ given by the
warp-derived mask, and within $1\!-\!M$ updates 3D patches of latent
voxels jointly. The patch partition is justified by the order-2-Markov
locality validation of Section~\ref{sec:appendix-locality-validation}
and is what makes the sub-state chain feasible at video scale: the
full-state chain on the $\sim\!10^5$ Wan~2.2 latent sites would mix
slowly per inner-iteration budget. SRVS does not
partition the state; the chain is full-state and the only spatial
structure is per-coordinate.

\emph{(iii) Mask source and mask placement.}
SRVS's eligibility mask is implicit (model self-disagreement) and
applied \emph{outside} the inner chain, as a mixing rule between
outer steps. \textit{GibbsLens}'s mask is external (camera-warp visibility)
and defines the partition $(M,1-M)$ that the inner chain itself
respects: observed sites are pinned, unobserved sites are refined.
The role the mask plays in the algorithm is categorically different
in the two methods.

\emph{(iv) Convergence diagnostic.}
SRVS gates refinement by the lag-1 paired-probe self-disagreement
$U(z^{(k-1)},z^{(k)})<\tau$, which requires a noise-aware
$\tau(\sigma_k)$ to remain calibrated across the schedule.
\textit{GibbsLens} uses the per-patch $\Delta$-Welford running-variance
plateau detector Eq.~\eqref{eq:m-stable-mask} with a single
relative threshold $\kappa\!\in\!(0,1]$ that is scale-free at the
patch level and works unchanged across all outer steps; the per-patch
granularity matches the locality scale at which neighboring
coordinates' chain trajectories are correlated.

\emph{(v) Readout from the inner loop.}
SRVS commits the last iterate $\hat z_0^{(K_f)}$ as the inner-loop
output. \textit{GibbsLens} consumes the Polyak-averaged readout
$\bar z_0\!=\!J_{\max}^{-1}\sum_j\hat z_0^{(j)}$, which feeds the
FlowMatch Euler step a $\sqrt{\tau_{\rm int}/J_{\max}}$-times more
concentrated estimator of the stationary mean at one in-place
accumulator's cost (Section~\ref{sec:appendix-polyak}).

In short, the genuinely new theoretical contribution of
\textit{GibbsLens} is (i): the sub-state extension of Bengio's Theorem~1 with
a pinned exogenous context, giving the inner chain a proven
stationary distribution equal to the partial-observation conditional
posterior. The genuinely new algorithmic contribution is (ii): the
block-conditional Gibbs primitive on patches of the unobserved
support, made adaptive by the per-patch $\Delta$-Welford gate and
locality-justified at video scale. Neither contribution exists in
SRVS or in any prior generalized-DAE-based sampler we are aware of.

\section{Block-precision locality validation}
\label{sec:appendix-locality-validation}

\paragraph{Setup.}
Stack the $N$ VAE-encoded videos into a tensor of shape $(N,C,L,H,W)$,
with $C$-channel voxel
$z^{(n)}_{l,h,w}\!\in\!\mathbb{R}^{C}$ at grid position $(l,h,w)$ of
video $n$. Following
Section~\ref{sec:method-locality-gibbs}, we work one axis at a time.
For axis $\alpha\!\in\!\{L,H,W\}$, an \emph{axis-aligned line stack}
is built by permuting $\alpha$ to slot~$1$ and the channel dim to
slot~$2$, then merging the two off-axis dimensions and the video
index into one row index:
\[
(N,C,L,H,W)\;\;\xrightarrow{\;\text{permute + flatten off-axis}\;}\;\;
\bigl(N_s,\,|\alpha|,\,C\bigr),
\qquad N_s\,=\,N\!\cdot\!\prod_{\alpha'\neq\alpha}|\alpha'|.
\]
Each row of this $\bigl(N_s,|\alpha|,C\bigr)$ tensor is one
\emph{line}
$\mathbf X\!=\!(X_1,\dots,X_{|\alpha|})\!\in\!\mathbb{R}^{|\alpha|\times C}$:
the $|\alpha|$ voxels swept along $\alpha$ at one (video, off-axis
location) pair, with $X_\beta\!\in\!\mathbb{R}^{C}$ the voxel at
axis-position $\beta$. Treating the $N_s$ rows as i.i.d.\ draws of
the joint $\mathbf X$ --- under a working assumption of stationarity
along the off-axis dimensions --- and flattening the last two
dimensions gives the design matrix
$\mathbf F\!\in\!\mathbb{R}^{N_s\times|\alpha|C}$ that the precision
pipeline below consumes. For the benchmarked
Wan~2.2 latent at $49$ frames, $704{\times}1280$, $C\!=\!48$, this
yields $N_s\!\in\![2.9\!\times\!10^4,\,1.8\!\times\!10^5]$ depending
on axis, comfortably $N_s\gg |\alpha|C$.

\paragraph{Block precision and conditional independence.}
Form the centered sample covariance
\[
\widehat\Sigma \;=\; \tfrac{1}{N_s-1}\,\mathbf F^\top\mathbf F\;\in\;\mathbb R^{|\alpha|C\times |\alpha|C},
\]
where $\mathbf F\in\mathbb R^{N_s\times |\alpha|C}$ stacks the $N_s$
flattened lines after per-coordinate standardization. Add a tiny relative
ridge
$\widehat\Sigma_\lambda=\widehat\Sigma+\lambda\,\mathrm{tr}(\widehat\Sigma)/(|\alpha|C)\,I$
($\lambda=10^{-6}$) for positive-definiteness and form the precision
$\widehat\Omega = \widehat\Sigma_\lambda^{-1}$ via numerically stable
Cholesky factorization. Reshape $\widehat\Omega$ into a
$|\alpha|\!\times\!|\alpha|$ array of $(C\!\times\!C)$ blocks
$\{\widehat\Omega_{\beta\gamma}\}$ indexed by axis positions
$\beta,\gamma\!\in\!\{1,\dots,|\alpha|\}$. The classical Gauss--Markov
result \citep{rogers2000diffusions} extended block-wise reads
\begin{equation}
\label{eq:appendix-locality-omega}
    X_\beta\;\perp\!\!\!\perp\;X_\gamma\,\bigm|\,X_{-\beta\gamma}
    \;\Longleftrightarrow\;
    \widehat\Omega_{\beta\gamma}=0_{C\times C},
\end{equation}
where $X_{-\beta\gamma}=\{X_r:r\neq \beta,\gamma\}$ is the rest of the line.

\paragraph{Block partial correlation via canonical correlations.}
The natural unit-free statistic is the block partial correlation
\[
\widehat R_{\beta\gamma} \;=\; -\,\widehat\Omega_{\beta\beta}^{-1/2}\,\widehat\Omega_{\beta\gamma}\,\widehat\Omega_{\gamma\gamma}^{-1/2}\;\in\;\mathbb R^{C\times C},
\]
whose singular values $\rho_1\ge\dots\ge\rho_C\ge0$ are the canonical
partial correlations between $X_\beta$ and $X_\gamma$ after regressing out
$X_{-\beta\gamma}$. The top canonical correlation
$\rho_1(\widehat R_{\beta\gamma})\in[0,1]$ measures the strongest linear
conditional dependence; vanishing of all $\rho_i$ recovers the
conditional-independence test of
Eq.~\eqref{eq:appendix-locality-omega}. On the diagonal $\beta\!=\!\gamma$,
$\widehat R_{\beta\beta}=-I$ and all singular values equal $1$ (the source
of the $|\alpha|\!\times\!|\alpha|$ heatmap diagonal in
Figure~\ref{fig:heatmap}).

\paragraph{Decay metric.}
For each axis we report the cumulative tail mass of the squared top
canonical partial correlation,
\begin{equation}
\label{eq:appendix-eta}
    \eta(r)
    \;=\;
    \frac{\sum_{|\beta-\gamma|>r}\bigl(\rho_1(\widehat R_{\beta\gamma})\bigr)^2}
         {\sum_{\beta\neq \gamma}\bigl(\rho_1(\widehat R_{\beta\gamma})\bigr)^2}\;\in\;[0,1].
\end{equation}
$\eta(r)$ is the fraction of squared off-diagonal partial-correlation mass
at distance greater than $r$; sharp locality $\Leftrightarrow$
$\eta(r)\to 0$ for some modest $r$.

\paragraph{Noise floor.}
A finite-sample top canonical correlation between two unrelated
$C$-vectors averages to
$\mathbb E[\rho_1]\!\approx\!2\sqrt{C/N_s}$ rather than zero. For our axes
this gives floors of $0.033$ ($L$), $0.061$ ($H$), $0.082$ ($W$); observed
far-distance values are $0.07$--$0.10$, within a factor of $\sim\!2$ of
the theoretical prediction. Most residual mass at large distances is
sampling noise rather than true long-range coupling.

\paragraph{Reading Figure~\ref{fig:heatmap}.}
Each axis gives a $|\alpha|\!\times\!|\alpha|$ heatmap of
$\rho_1(\widehat R_{\beta\gamma})$. A clean local axis satisfies
(i)~$\rho_1\ge 0.5$ at $|\beta-\gamma|=1$, (ii)~drop to $\le\!1/3$ of that
value by $|\beta-\gamma|=2$, (iii)~floor near $2\sqrt{C/N_s}$ for
$|\beta-\gamma|\ge 3$. Figure~\ref{fig:heatmap} satisfies all three on
each of $L,H,W$ for the Wan~2.2 latent, justifying the choice of patch
radius $\ge 2$ voxels in Section~\ref{sec:method-locality-gibbs}.

\begin{figure}[p]
  \centering
  \subcaptionbox{$\sigma_t=0.1$\label{fig:heatmap-s01}}{%
    \includegraphics[width=0.78\linewidth]{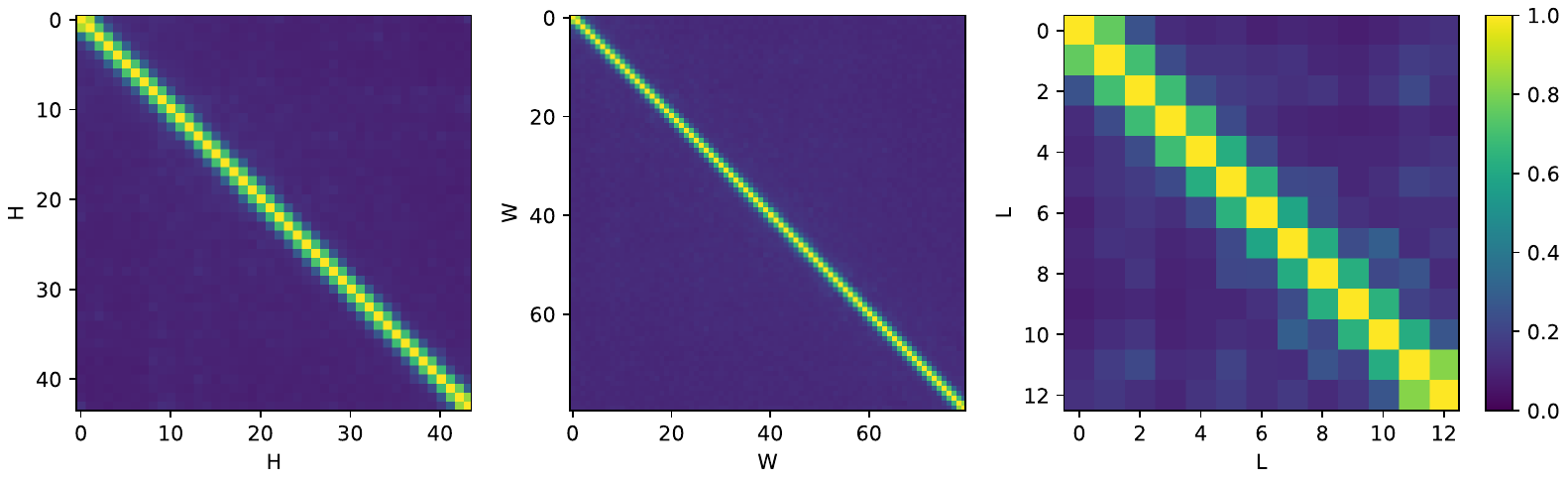}}\\[0.5ex]
  \subcaptionbox{$\sigma_t=0.3$\label{fig:heatmap-s03}}{%
    \includegraphics[width=0.78\linewidth]{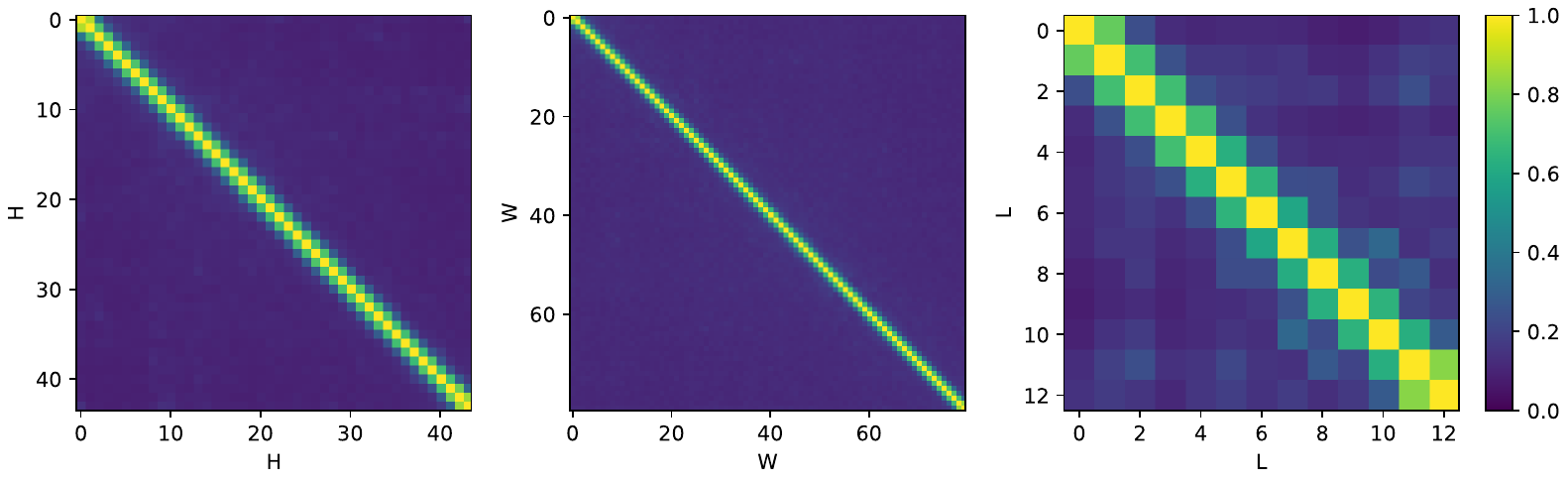}}\\[0.5ex]
  \subcaptionbox{$\sigma_t=0.5$\label{fig:heatmap-s05}}{%
    \includegraphics[width=0.78\linewidth]{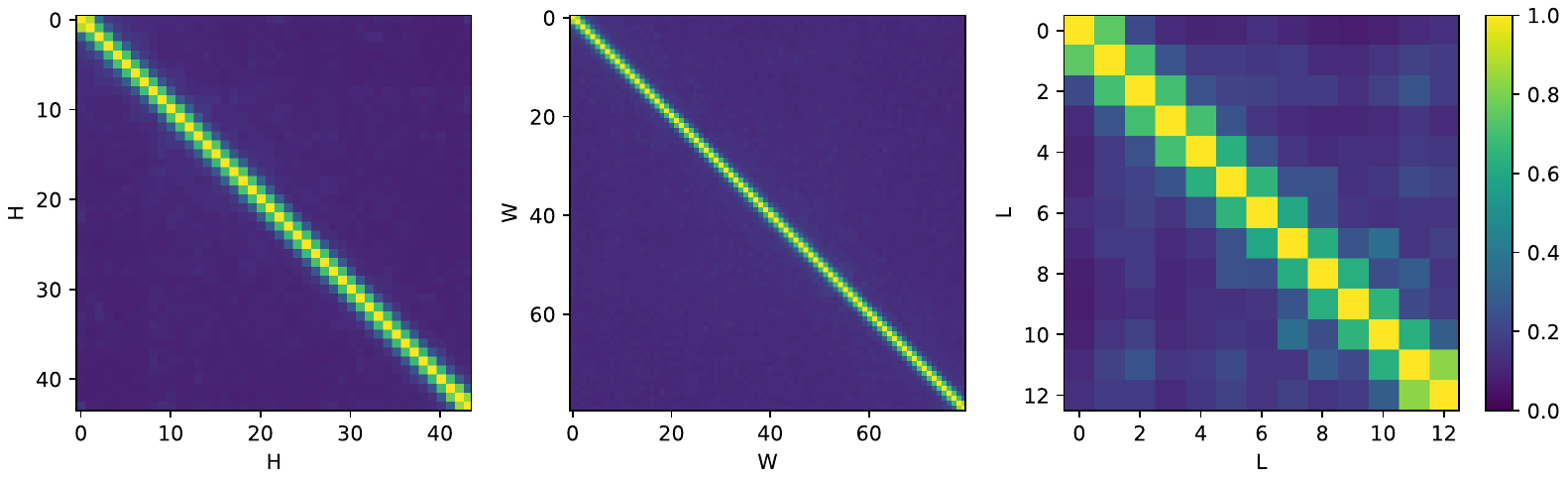}}\\[0.5ex]
  \subcaptionbox{$\sigma_t=0.7$\label{fig:heatmap-s07}}{%
    \includegraphics[width=0.78\linewidth]{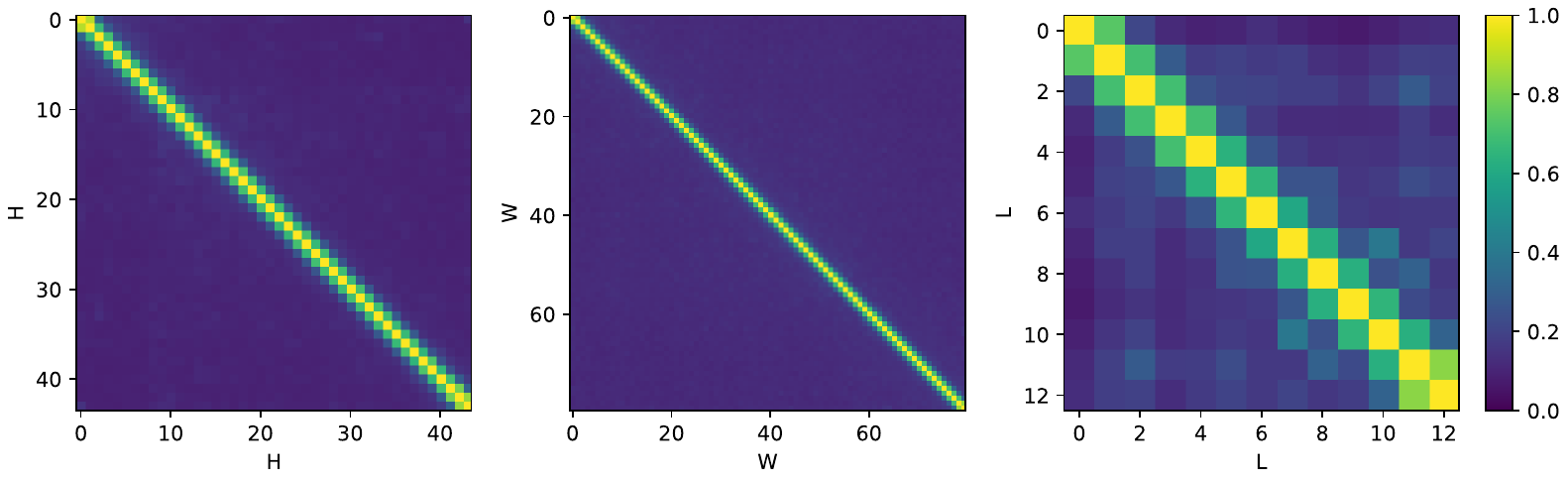}}\\[0.5ex]
  \subcaptionbox{$\sigma_t=0.9$\label{fig:heatmap-s09}}{%
    \includegraphics[width=0.78\linewidth]{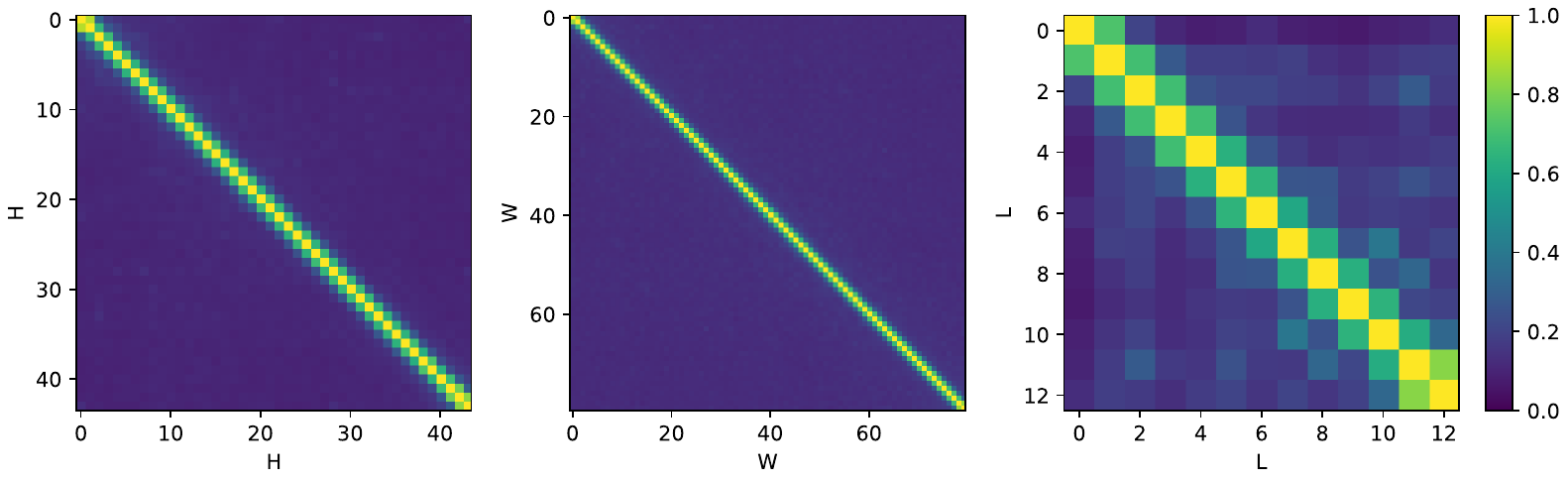}}
  \caption{Top canonical partial correlation
  $\rho_1(\widehat R_{\beta\gamma})$ on the model's clean prediction
  $\hat z_0(z_t,t,c)$ along the $H$, $W$, $L$ axes (left to right) at
  five noise levels. The diagonal band stays sharp at every $\sigma_t$
  with off-diagonal mass concentrated within $|\beta-\gamma|\!\le\!2$,
  confirming that the locality structure of Figure~\ref{fig:heatmap}
  is preserved across the sampling schedule.}
  \label{fig:heatmap-sigma}
\end{figure}

\paragraph{Locality across the noise schedule.}
The diagnostic above applies to clean VAE-encoded latents. We extend it
to the model's clean prediction $\hat z_0(z_t,t,c)$ at non-zero noise
levels, since the inner-loop perturb-and-redenoise of
Eqs.~\eqref{eq:m-perturb}--\eqref{eq:m-redenoise} operates in
clean-prediction space at the current noise level $\sigma_t$. We
generate $200$ videos by running the official Wan~2.2 flow-matching
sampler from random initial noise $z_1\!\sim\!\mathcal N(0,I)$
down to $z_0$, and record the full sampling trajectory $\{z_t\}$.
At every recorded denoising step we evaluate
$\hat z_0(z_t,t,c)=z_t-\sigma_t u_\theta(z_t,t,c)$ and run the
block-precision pipeline of Eq.~\eqref{eq:appendix-locality-omega} on
the resulting $\hat z_0$ tensors at
$\sigma_t\in\{0.1,0.3,0.5,0.7,0.9\}$.
Figure~\ref{fig:heatmap-sigma} shows the resulting per-axis heatmaps
at all five noise levels: the diagonal band stays sharp at every
$\sigma_t$, with off-diagonal mass dropping into the finite-sample
noise floor by distance $|\beta-\gamma|\!=\!2$--$3$ on each of
$L,H,W$, identical to the clean-latent diagnostic of
Figure~\ref{fig:heatmap}. The locality structure is therefore
preserved across the entire sampling trajectory, not just on clean
encoded latents, and the block-conditional Gibbs partition of
Section~\ref{sec:method-locality-gibbs} is valid at every active
\textit{GibbsLens} outer step.

\section{Polyak-averaged readout: compatibility with the FlowMatch Euler step}
\label{sec:appendix-polyak}

The write-back Eq.~\eqref{eq:m-writeback} consumes the running mean
$\bar z_0 = J_{\max}^{-1}\sum_{j=1}^{J_{\max}} \hat z_0^{(j)}$ of the
inner clean-prediction iterates, not the last iterate. The variance
reduction argument is the standard MCMC iterate-averaging
result~\citep{polyak1992acceleration}: at stationarity the running mean has
variance $\sigma_\pi^2\,\tau_{\rm int}/J + \mathcal O(J^{-2})$ vs.\ the
last iterate's $\sigma_\pi^2$, where $\tau_{\rm int}$ is the integrated
autocorrelation time. We add only the observation that this estimator is
the right input to feed into our outer integrator. The FlowMatch Euler
update
$z_{k+1}=z_k+(\sigma_{k+1}-\sigma_k)\,(z_k-\hat z_0^{\rm final})/\sigma_k$
is \emph{linear} in $\hat z_0^{\rm final}$, so any extra variance in the
readout is injected into the latent $z_{k+1}$ without an averaging or
projection step to absorb it; conversely, the bias/variance decomposition
of the Polyak-averaged readout passes through the linear update unchanged.
Consuming the running mean therefore hands the integrator a
$\sqrt{\tau_{\rm int}/J_{\max}}$-times more concentrated estimator of the
stationary mean than the last iterate would, at the cost of one in-place
accumulator per outer step --- a strict improvement under the linearity
that FlowMatch already provides.

\section{Toy-example details}
\label{sec:appendix-toy}

We have provided the code of the toy example for reproducing all the figures in the paper.

\textbf{Setup.}
The unconditional density is a 2D 8-square checkerboard: eight unit
squares centered at
$\{(i+0.5, j+0.5):i,j\in\{-2,-1,0,1\}, (i+j)\,\mathrm{mod}\,2=0\}$,
sampled uniform-area within each square. We train a 6-layer
flow-matching MLP with sinusoidal time embedding and $\mathrm{SiLU}$
activations on this density for $20{,}000$ iterations (batch $512$,
AdamW with $\mathrm{lr}=2{\times}10^{-3}$ and weight decay $10^{-4}$,
cosine schedule); the trained MLP attains a manifold-hit rate above
$99\%$ on unconditional samples, certifying that any conditional
shortfall observed below is due to the conditional sampler rather than
a poorly trained prior. We then evaluate conditional sampling under a
noisy partial observation of the first coordinate
$y\!\sim\!\mathcal N(z_0[1],\,\sigma_y^2)$ with $\sigma_y\!=\!0.2$,
anchored at $y_{\rm obs}\!\approx\!0.5$ where the vertical constraint
line intersects exactly two of the eight filled squares, yielding a
binary mode-coverage question with two well-separated conditional
modes.

\textbf{Evaluation.}
We compare \textit{GibbsLens} against DPS~\citep{chung2022diffusion}
(compute-native at $\mathrm{NFE}\!=\!100$, $1$ forward $+$ $1$ backward
per step at $50$ outer steps) and TFG-UGD~\citep{ye2024tfg} at
$\mathrm{NFE}\!=\!250$ (configuration $N_{\rm recur}\!=\!2$,
$N_{\rm iter}\!=\!1$, $\mu\!=\!0.5$, $\rho\!=\!0.5$, giving $5$ calls per
step at $50$ outer steps). \textit{GibbsLens} with inner budget $J\!=\!4$ also
costs $5$ calls per step ($1$ outer forward $+$ $4$ inner forwards), so
it is compute-matched to TFG-UGD at $\mathrm{NFE}\!=\!250$. For the
matched sweep of Figure~\ref{fig:toy-b}, we vary $J\!\in\!\{0,\dots,16\}$
for \textit{GibbsLens} and $N_{\rm recur}\!\in\!\{1,\dots,6\}$ for TFG-UGD,
plotting total NFE on the horizontal axis.

Each method is evaluated with $10$ seeds and $500$ samples per seed
($5{,}000$ samples total). The posterior-hit rate is the fraction of
samples falling \emph{both} inside one of the two filled squares
intersected by the constraint line \emph{and} within $|x_1-y_{\rm obs}|<0.5$
of the anchor. A sample on the constraint line but outside the filled
squares contributes zero to posterior-hit but may still contribute to
manifold-hit (a sanity check that the sampler did not drift entirely off
the data manifold).

\textbf{Results.}
Figure~\ref{fig:toy-c} bins the inner iterations $j$ by noise band and
plots $|\Delta_W^{(j)}|$ (averaged over the population of toy chains) on
a shared horizontal axis with the posterior-hit rate of
Figure~\ref{fig:toy-b}; the two curves plateau over the same $j$-range,
the empirical content of the claim that $|\Delta_W^{(j)}|\!\to\!0$ is a
reliable ground-truth-free trigger for the early-freeze gate. Disabling
the inner loop ($J\!=\!0$) reduces \textit{GibbsLens} to DPS-like behaviour:
with the inner Gibbs chain in
Eqs.~\eqref{eq:m-perturb}--\eqref{eq:m-redenoise} disabled, only the
outer soft pull of Eq.~\eqref{eq:m-vctrl} remains, which at the
$\sigma_y\!=\!0.2$ used here is the DPS gradient evaluated at the
flow-matching Tweedie mean (Appendix~\ref{sec:appendix-surrogates}). The
\textit{GibbsLens} sample cloud then collapses onto the same constraint-line
support that DPS produces in Figure~\ref{fig:toy-a}, with mode coverage
governed entirely by the outer pull --- the clean ablation that isolates
the inner loop as the source of the mode-coverage gain.

\section{\textit{GibbsLens} implementation details}
\label{sec:appendix-impl}

\subsection{Hyperparameter Settings}
\label{sec:hyperparams}

Table~\ref{tab:hyperparam} lists the necessary hyperparameters used
in our main experiments with \textit{GibbsLens} on the Wan~2.2 backbone.
The discretized FlowMatch sampler runs for $50$ outer steps,
indexed so that step $0$ corresponds to pure noise and step $50$
to the clean sample. The camera-control window covers outer steps
$t_s\!=\!2$ through $t_e\!=\!5$, the highest-noise interval where
the warp constraint is most informative; at each step within this
window the inner-loop refinement runs with budget
$J_{\max}\!=\!10$, with the $\Delta$-Welford freeze gate
(Section~\ref{sec:method-locality-gibbs}) allowing per-patch early
termination so that the realised inner iteration count averages
below $J_{\max}$. We use classifier-free guidance with scale~$4.0$.

\begin{table}[htbp]
\centering
\small
\setlength{\tabcolsep}{6pt}
\renewcommand{\arraystretch}{1.15}
\caption{Hyperparameter settings for the main experiments
(\textit{GibbsLens} on the Wan~2.2 backbone).}
\label{tab:hyperparam}
\begin{tabular}{lcl}
\toprule
\textbf{Hyperparameter} & \textbf{Value} & \textbf{Description} \\
\midrule
Sampling steps & $50$  & Total discretized FlowMatch outer steps \\
$t_s$          & $2$   & First outer step of the camera-control window \\
$t_e$          & $5$   & Last outer step of the camera-control window \\
$J_{\max}$     & $10$  & Inner-loop budget per outer step in the camera-control window \\
CFG scale      & $4.0$ & Classifier-free guidance scale \\
\bottomrule
\end{tabular}
\end{table}

\section{Datasets and evaluation protocol}
\label{sec:appendix-eval}

\textbf{RealEstate10K.} We follow the RealEstate10K~\citep{zhou2018stereo} train/test split.
\textbf{Sampling.} We sample $200$ scenes uniformly at random, with a
fixed seed for reproducibility. For each
scene we treat the first frame as the I2V conditioning input and the
remaining $48$ frames as the held-out reference under the source camera
trajectory. Reference-based metrics (LPIPS, SSIM) are computed
between the generated video and the held-out reference frames after
matching frame counts. The target trajectory
is the one parsed from the dataset's pose annotations; it is fed to every
method that takes camera parameters explicitly, and used to render the
warp for every training-free method that takes a warp.

\textbf{DAVIS controlled-trajectory synthesis} We take the DAVIS-2017 challenge set, restricted to the $28$
sequences with continuous moderate-to-large object motion. For each scene we synthesize three target trajectories of
varying angular extent --- approximately $\pm 15^\circ$,
$\pm 30^\circ$, and $\pm 45^\circ$ horizontal sweep around the source
camera --- giving $84$ controlled-trajectory videos in total. Each
trajectory is realized as a smooth $49$-frame sequence of $6$-DOF camera
poses; they are converted into a warp by the same depth-based
lift-and-reproject pipeline used by every training-free baseline
(see Appendix~\ref{sec:appendix-impl}).
DAVIS lacks ground-truth novel views under the synthesized trajectories,
so we do not report PSNR/LPIPS/SSIM and instead use FVD against the source
distribution and the camera-pose-derived trajectory metrics on
re-extracted poses (Section~\ref{sec:appendix-eval-traj}).

\subsection{Camera trajectory metrics}
\label{sec:appendix-eval-traj}

For every generated video we re-extract the $49$-frame camera trajectory
with Mega-SAM~\citep{li2025megasam}, a recent monocular SLAM tracker,
without any tuning. The re-extracted trajectory is aligned to the
target trajectory by a single similarity transform (translation,
rotation, scale) and then evaluated by:
\begin{itemize}
\item \textbf{ATE} (Absolute Trajectory Error): per-frame root-mean-square
translational error after alignment, in centimeters under the dataset's
metric scale.
\item \textbf{RRE} (Relative Rotation Error): per-pair root-mean-square
rotational error between consecutive frames, in degrees.
\item \textbf{RTE} (Relative Translation Error): per-pair root-mean-square
translational error between consecutive frames, in centimeters.
\end{itemize}
Failures of Mega-SAM (re-extraction divergence on extreme dynamic content)
are extremely rare on RealEstate10K and occur on $<\!3\%$ of DAVIS
trajectories; affected videos are dropped uniformly across all methods to
keep the comparison fair, with the drop count reported per row.

\subsection{Perceptual metrics}
\label{sec:appendix-eval-perc}

\textbf{FVD.} Standard Fr\'echet Video
Distance~\citep{unterthiner2018towards} computed against a held-out
reference distribution per dataset: source DAVIS clips on DAVIS, source
RealEstate10K clips on RealEstate10K. Features come from the I3D
backbone trained on Kinetics-400; we use the public implementation with
default centering.

\textbf{FVMD.} Fr\'echet Video Motion Distance~\citep{liu2024fvmd}
compares distributions of explicit motion features derived from keypoint
tracks. It complements FVD by assessing motion consistency; lower is better.

\textbf{VisionReward.} We report the overall video score
(VR-Overall) from VisionReward~\citep{xu2026visionreward}, using the
released weighted scoring implementation. Scores lie in
$[-0.21645,\,0.21645]$, with higher values indicating stronger alignment
with learned human preferences. This automatic metric complements,
rather than replaces, human evaluation.

\textbf{CLIP-v / CLIP-f.} CLIP-based video coherence (CLIP-v,
mean cosine similarity of frame-CLIP embeddings within the same video)
and frame-text alignment (CLIP-f, mean cosine similarity between each
frame's CLIP image embedding and the scene caption's CLIP text
embedding). These are the two text-aligned consistency metrics commonly
used.

\textbf{LPIPS.} Learned Perceptual Image Patch
Similarity~\citep{zhang2018lpips}: a reference-based perceptual metric
that compares deep features (we use the standard AlexNet backbone)
between two RGB frames, calibrated to human perceptual judgments. We
report LPIPS only on RealEstate10K, where each generated frame can be
compared to a held-out ground-truth view at the corresponding pose;
DAVIS provides no such reference for novel-trajectory views. Lower is
better.

\textbf{SSIM.} Structural Similarity Index~\citep{wang2004ssim}: a
classical reference-based metric that compares the luminance, contrast,
and structural components of two RGB frames in a sliding window. As
with LPIPS, we report SSIM only on RealEstate10K against the held-out
ground-truth view, computed per frame and averaged over the video.
Higher is better.

\section{Additional quantitative results}
\label{sec:appendix-quant}

\subsection{Inference speed}
\label{sec:appendix-runtime}

Table~\ref{tab:time} reports inference speed in frames/min on a single
NVIDIA~A40 GPU at each method's native operating resolution. \textit{GibbsLens}
runs at $7.13$ frames/min at the full $720{\times}1280$ resolution
shared by every training-free baseline---the second-fastest method
overall, behind only TTM ($12.99$), which is bare hard-replacement
without any inner-loop refinement and therefore represents the natural
runtime ceiling. As quantified in Table~\ref{tab:compute-cost}, the
inner-loop pseudo-Gibbs refinement increases runtime by $81.9\%$
relative to TTM, with less than $0.01$ GiB of additional peak GPU memory.
Nevertheless, \textit{GibbsLens} remains faster than the
remaining training-free baselines (WorldForge $3.92$, RePaint $6.23$,
Coarse-Guided $6.97$) and faster than every training-based controller,
even though the training-based methods operate at strictly lower
resolution (ReCamMaster $5.55$ at $480{\times}832$, TrajectoryCrafter
$4.42$ at $384{\times}672$, TrajectoryAttention $1.36$ at
$576{\times}1024$). Combined with the FVD gains of
Tables~\ref{tab:re10k-quant}--\ref{tab:davis-quant}, this places
\textit{GibbsLens} on the favourable side of the quality--speed Pareto
frontier.

\begin{table}[h]
\centering
\small
\setlength{\tabcolsep}{8pt}
\renewcommand{\arraystretch}{1.15}
\caption{Inference speed of training-free and training-based
camera-control baselines and ours, reported at each method's native
operating resolution on a single NVIDIA~A40 GPU. $^\dagger$ marks
training-based methods.}
\label{tab:time}
\begin{tabular}{lccc}
\toprule
 & \textbf{Resolution} & \textbf{\shortstack{Inference speed\\(frames/min)}} & \textbf{\shortstack{Base video\\model}}\\
\midrule
WorldForge~\citep{song2025worldforge}                    & $720{\times}1280$ & $3.92$  & Wan~2.2 \\
TTM~\citep{singer2025time}                           & $720{\times}1280$ & $\mathbf{12.99}$ & Wan~2.2 \\
Coarse-Guided~\citep{wang2026coarse}                 & $720{\times}1280$ & $6.97$  & Wan~2.2 \\
RePaint~\citep{lugmayr2022repaint}                       & $720{\times}1280$ & $6.23$  & Wan~2.2 \\
\midrule
TrajectoryAttention$^\dagger$~\citep{xiao2024trajectory} & $576{\times}1024$ & $1.36$  & Stable Video Diffusion \\
TrajectoryCrafter$^\dagger$~\citep{yu2025trajectorycrafter}   & $384{\times}672$  & $4.42$  & CogVideoX \\
ReCamMaster$^\dagger$~\citep{bai2025recammaster}         & $480{\times}832$  & $5.55$  & Wan~2.1 \\
\midrule
\textbf{\textit{GibbsLens} (Ours)}   & $720{\times}1280$ & $\underline{7.13}$ & Wan~2.2 \\
\bottomrule
\end{tabular}
\end{table}

\section{Additional qualitative results}
\label{sec:appendix-quali}

\subsection{Failure-mode catalogue}
\label{sec:appendix-quali-failures}



Figure~\ref{fig:appendix-failure-depth} compares the warped guidance
video (with depth-induced wrong warping), the ground-truth target
trajectory, and our \textit{GibbsLens} output. Although the inner-loop
refinement still produces visually plausible inpainting on the
unobserved complement (the generated frames are seam-free and
natural-looking), the output tracks the wrongly warped guidance rather
than the ground-truth trajectory. This is the expected behaviour under
our partial-observation framing (Section~\ref{sec:method}): the warped
video supplies the posterior evidence, and once that evidence is
geometrically inconsistent the pretrained prior cannot correct it,
since the prior encodes what plausible video looks like rather than
what the correct camera trajectory should be. Robustness to severe
depth misestimation therefore requires either a more reliable depth
estimator upstream or a separate trajectory prior beyond the visual
prior $p(z_0)$.
\begin{figure}[!ht]
  \centering
  \includegraphics[width=0.9\linewidth]{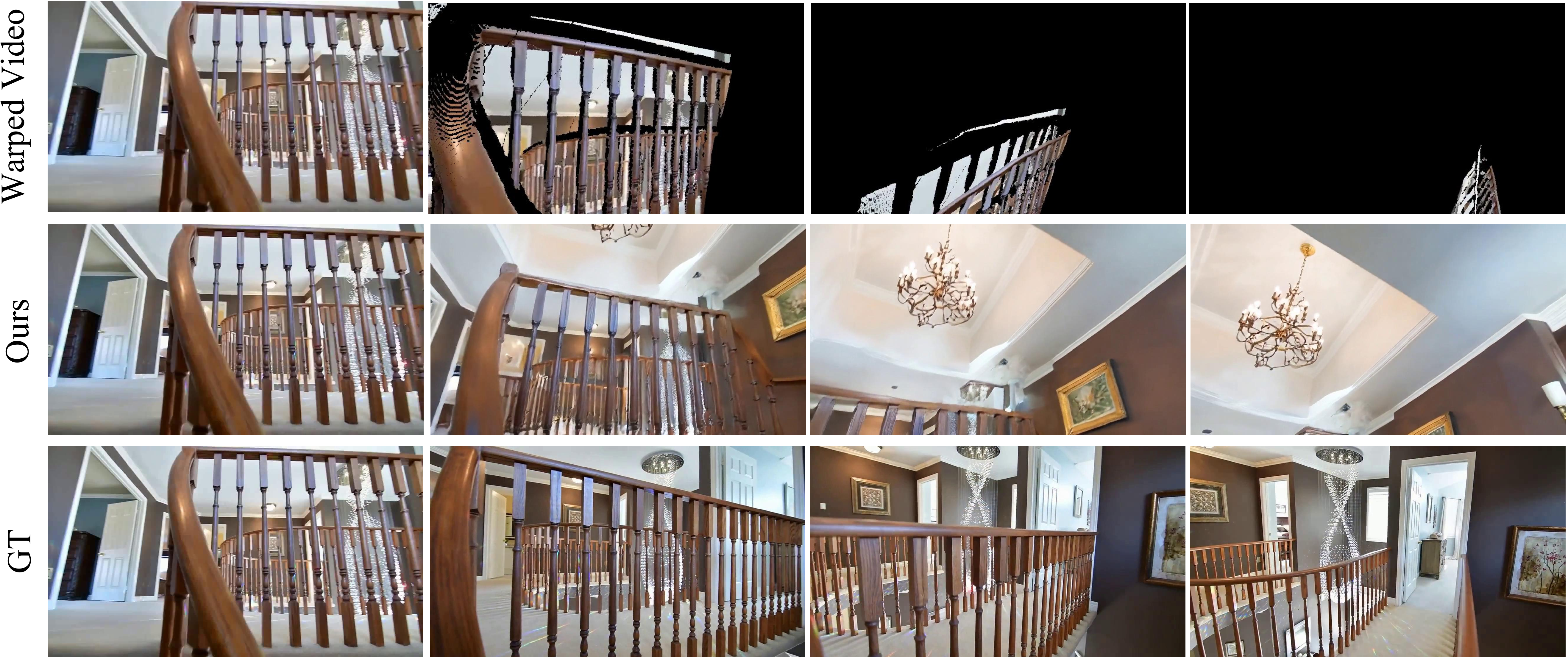}
  \caption{\textbf{Failure modes.} Severe depth estimation error
  produces wrong warping, leads to bad camera control. }
  \label{fig:appendix-failure-depth}
\end{figure}

\subsection{RealEstate10K}
\label{sec:appendix-quali-re10k}

Figure~\ref{fig:appendix-re10k-gallery} extends the main-paper
qualitative panel (Figure~\ref{fig:real10k_exp}) with additional
held-out RealEstate10K scenes spanning a broader range of trajectory
extents. 
\begin{figure}[!ht]
  \centering{%
    \includegraphics[width=1\linewidth]{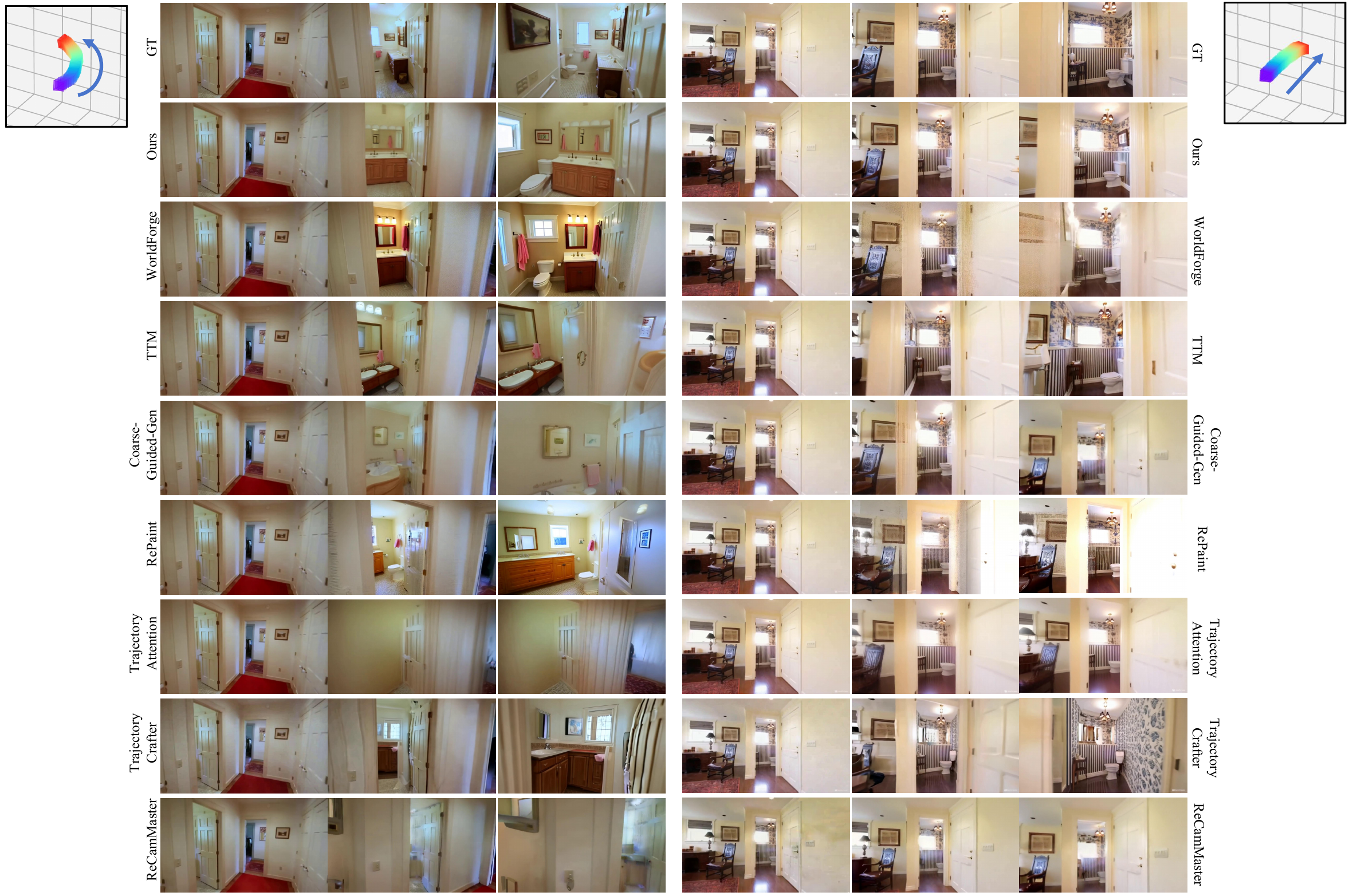}}
  \caption{Additional RealEstate10K qualitative comparisons.}
  \label{fig:appendix-re10k-gallery}
\end{figure}

\subsection{DAVIS}
\label{sec:appendix-quali-davis}

Figure~\ref{fig:appendix-davis-2} extends the main-paper DAVIS
qualitative panel (Figure~\ref{fig:davis-qualitative}) with additional
dynamic scenes under larger trajectory sweeps. 

\begin{figure}[!ht]
  \centering{%
    \includegraphics[width=1\linewidth]{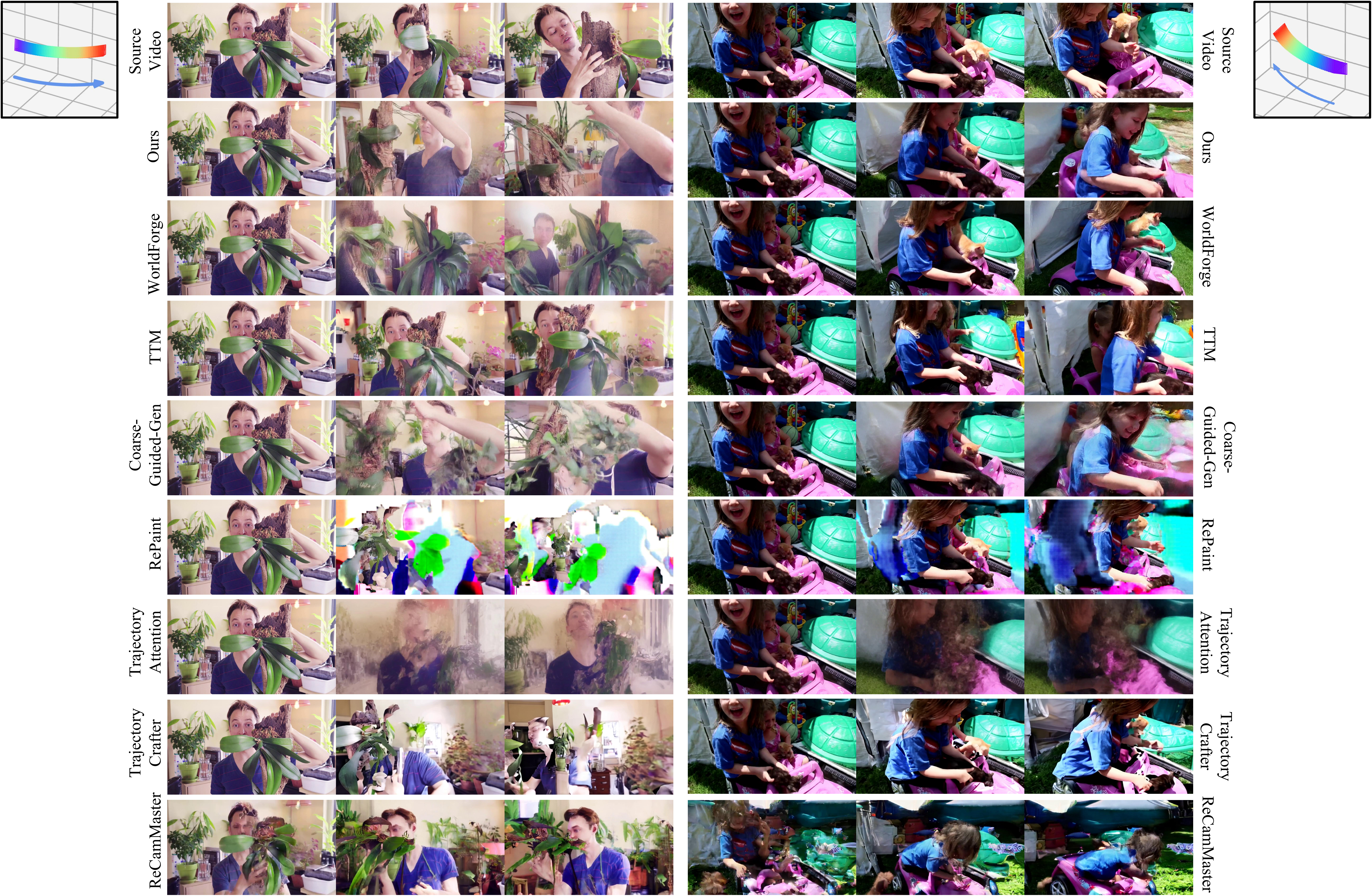}}
    \\
    [0.5ex]
{%
    \includegraphics[width=1\linewidth]{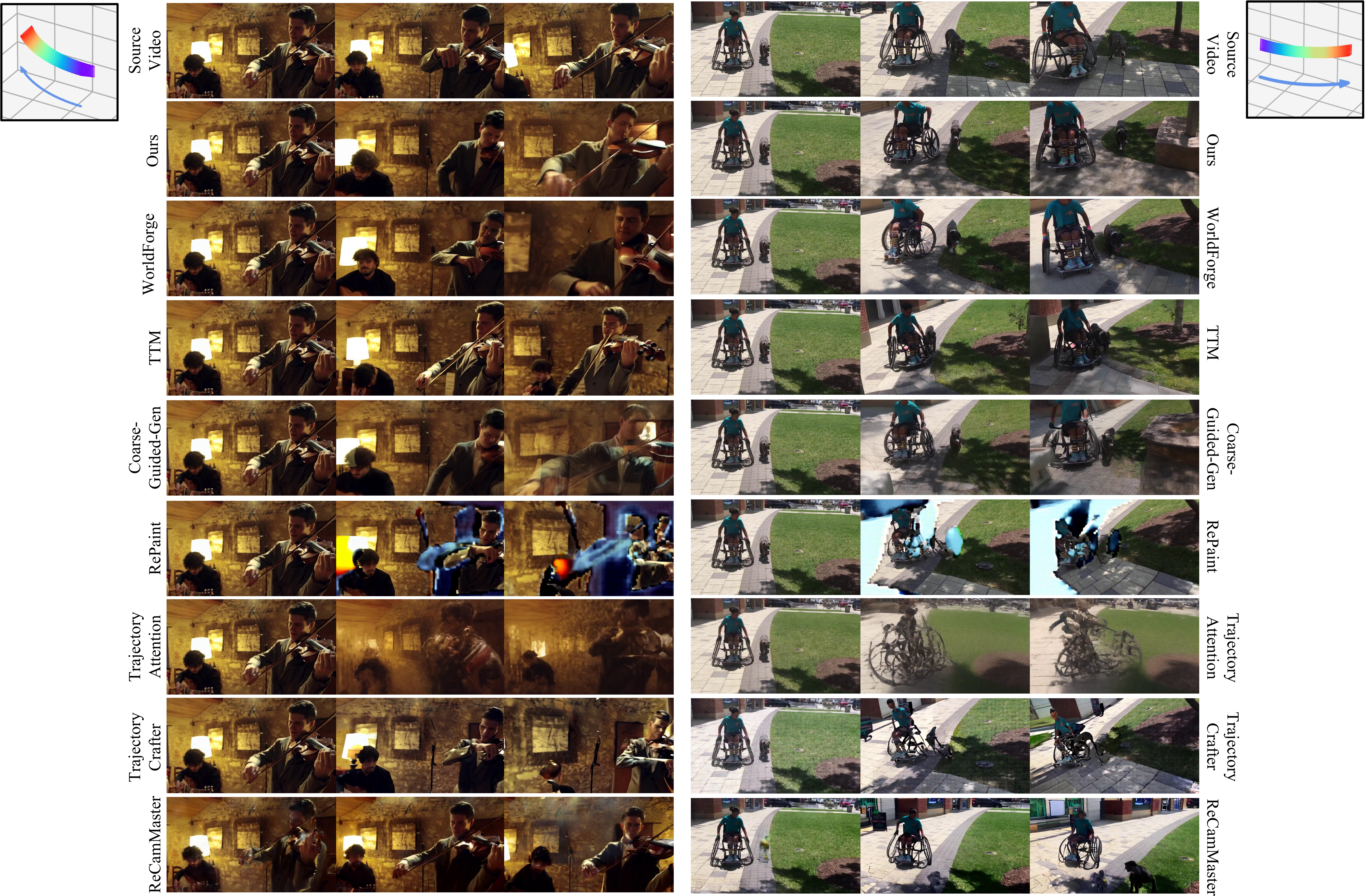}}
    \caption{ More comparison with state-of-the-art methods on Davis.}
  \label{fig:appendix-davis-2}
\end{figure}
\section{Broader impacts}
\label{sec:appendix-broader-impacts}

\paragraph{Positive impacts.}
Improved camera control for pretrained video generators has constructive
applications in filmmaking previsualization, simulation and synthetic
data for embodied AI, virtual- and augmented-reality content authoring,
and education. Because \textit{GibbsLens} is training-free and adds no new
generative capability on top of the base model, it lowers the engineering
barrier to high-quality camera-controlled video without requiring new
large-scale training, which improves accessibility for users without
the compute to fine-tune video generators.

\paragraph{Negative impacts.}
Higher-fidelity camera control on pretrained video generators can also
make synthetic video more convincing, which inherits the well-known
misuse profile of modern video generative models---non-consensual or
deceptive synthetic media, impersonation, and disinformation. \textit{GibbsLens}
itself does not introduce new training data or new generative capacity;
it operates as a sampler on top of an already-released backbone, so the
risks it raises are the risks of the underlying generator combined with
sharper trajectory control.

\paragraph{Mitigations.}
We rely on the safeguards of the underlying pretrained video generator
(content filters, license restrictions, watermarking where applicable),
release no new pretrained model or scraped dataset, and use only public
benchmarks (RealEstate10K, DAVIS) under their original terms.

\clearpage



\end{document}